\documentclass[conference]{IEEEtran}

\usepackage{cite}

\usepackage{latexsym}
\usepackage{comment}
\usepackage{lineno,hyperref}
\usepackage[utf8]{inputenc}
\usepackage[T1]{fontenc}
\usepackage{url}   
\usepackage{bbding}
\usepackage{pgfplots}

\newcommand{\boldup}{$\boldsymbol{\uparrow}$}
\newcommand{\bolddown}{$\boldsymbol{\downarrow}$}

\usepackage{microtype}      
\usepackage[dvipsnames,svgnames]{xcolor} 

\usepackage{nicefrac}       
\usepackage{bm}
\usepackage{amsmath}

\usepackage{amssymb}
\usepackage{amsfonts}       
\usepackage{eucal}

\usepackage{algorithmic}
\usepackage{algorithm}
\usepackage{textcomp}
\usepackage{adjustbox, tabularx}
\usepackage{makecell, booktabs} 
\makegapedcells

\usepackage{array, multirow}
\usepackage{arydshln}
\usepackage{svg}
\usepackage{caption}
\usepackage{subfig}
\usepackage{graphicx}
\graphicspath{{./Metadata/}}

\usepackage{todonotes}
\usepackage{amsthm}

\usepackage{xcolor}
\definecolor{Red}{HTML}{aa0000}
\definecolor{Green}{HTML}{00aa00}

\newtheorem{thm}{Theorem}
\newtheorem{lem}{Lemma}

\newtheorem{defn}{Definition}

\def\BibTeX{{\rm B\kern-.05em{\sc i\kern-.025em b}\kern-.08em
    T\kern-.1667em\lower.7ex\hbox{E}\kern-.125emX}}

\newcommand{\z}{\textcolor{gray!45}{0.00}}  
\renewcommand{\arraystretch}{1.08}

\begin{document}

\title{A Deep Latent Factor Graph Clustering\\ with Fairness-Utility Trade-off Perspective}

\author{
\IEEEauthorblockN{ Siamak Ghodsi \href{https://orcid.org/0000-0002-3306-4233}{\includegraphics[scale=0.012]{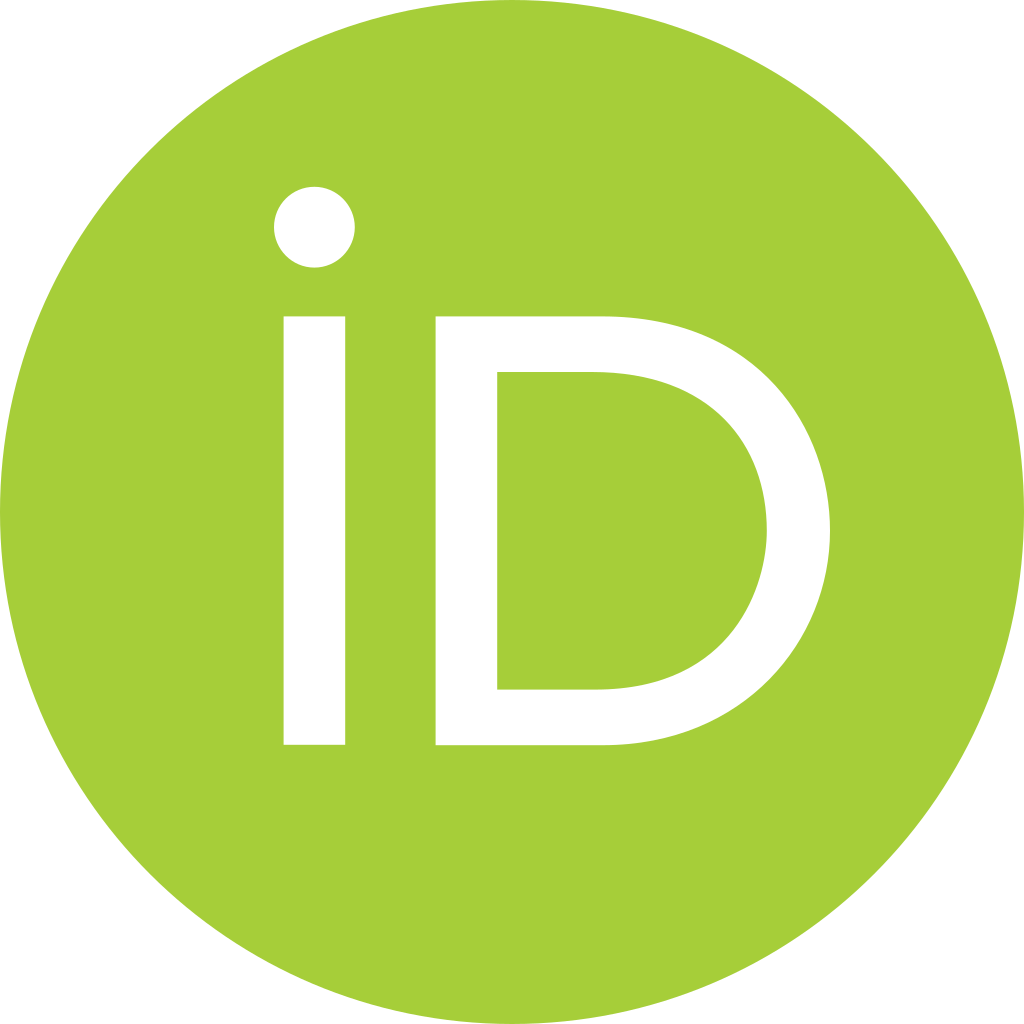}}}
\IEEEauthorblockA{\textit{L3S Research Center} \\
Hannover, Germany \\
ghodsi[at]l3s.de}
\and
\IEEEauthorblockN{Amjad Seyedi \href{https://orcid.org/0000-0003-2718-7146}{\includegraphics[scale=0.012]{orcid.png}}}
\IEEEauthorblockA{\textit{University of Mons} \\
Mons, Belgium \\
seyedamjad.seyedi[at]umons.ac.be}
\and
\IEEEauthorblockN{Tai Le Quy  \href{https://orcid.org/0000-0001-8512-5854}{\includegraphics[scale=0.012]{orcid.png}}}
\IEEEauthorblockA{\textit{University of Koblenz} \\
Koblenz, Germany \\
tailequy[at]uni-koblenz.de}
\and
\IEEEauthorblockN{Fariba Karimi \href{https://orcid.org/0000-0002-0037-2475}{\includegraphics[scale=0.012]{orcid.png}}}
\IEEEauthorblockA{\textit{TU Graz} \\
Graz, Austria \\
karimi[at]csh.ac.at}
\and
\IEEEauthorblockN{Eirini Ntoutsi \href{https://orcid.org/0000-0001-5729-1003}{\includegraphics[scale=0.012]{orcid.png}}}
\IEEEauthorblockA{\textit{Bundeswehr University} \\
Munich, Germany \\
eirini.ntoutsi[at]unibw.de}

}

\maketitle

\begin{abstract}
Fair graph clustering seeks partitions that respect network structure while maintaining proportional representation across sensitive groups, with applications spanning community detection, team formation, resource allocation, and social network analysis. Many existing approaches enforce rigid constraints or rely on multi-stage pipelines (e.g., spectral embedding followed by $k$-means), limiting trade-off control, interpretability, and scalability. We introduce \emph{DFNMF}, an end-to-end deep nonnegative tri-factorization tailored to graphs that directly optimizes cluster assignments with a soft statistical-parity regularizer. A single parameter $\lambda$ tunes the fairness--utility balance, while nonnegativity yields parts-based factors and transparent soft memberships. The optimization uses sparse-friendly alternating updates and scales near-linearly with the number of edges. Across synthetic and real networks, DFNMF achieves substantially higher group balance at comparable modularity, often dominating state-of-the-art baselines on the Pareto front. The code is available at \url{https://github.com/SiamakGhodsi/DFNMF.git}. 
\end{abstract}


\section{Introduction} \label{sec: intro}
Fair and trustworthy machine learning has advanced rapidly over the last decade~\cite{ntoutsi2020bias, DBLP:journals/widm/QuyRIZN22, DBLP:conf/bigdataconf/ZhaoFLGBSK24, DBLP:conf/ecir/TahmasebiME25}, yet it remains underexplored in graph learning—especially graph clustering~\cite{DBLP:journals/tkde/DongMWCL23, DBLP:journals/ethicsit/AlvarezCEFFFGMPLRSSZR24}.
Graph clustering partitions a network into cohesive groups and underpins applications such as community detection, team formation, resource allocation, and social network analysis \cite{DBLP:conf/ijcai/JuY0L00024}. Classical approaches maximize utility (e.g., modularity) or minimize cuts to uncover natural network structures, yet many real-world applications require balancing structural cohesion with demographic fairness.

Consider academic collaboration networks where funders increasingly mandate diverse team compositions. Purely structure-driven clustering may recover prolific yet homogeneous groups that fail diversity requirements. Similarly, educational institutions forming student project teams must balance social connections with equitable assignments (Figure~\ref{fig:teaser_fig}). These scenarios demand fair graph clustering—modifying natural community boundaries to achieve demographic balance while preserving meaningful network structure~\cite{DBLP:conf/pakdd/QuyFN23}.

\begin{figure}[!ht]%
    \centering
    \includegraphics[width=\linewidth]{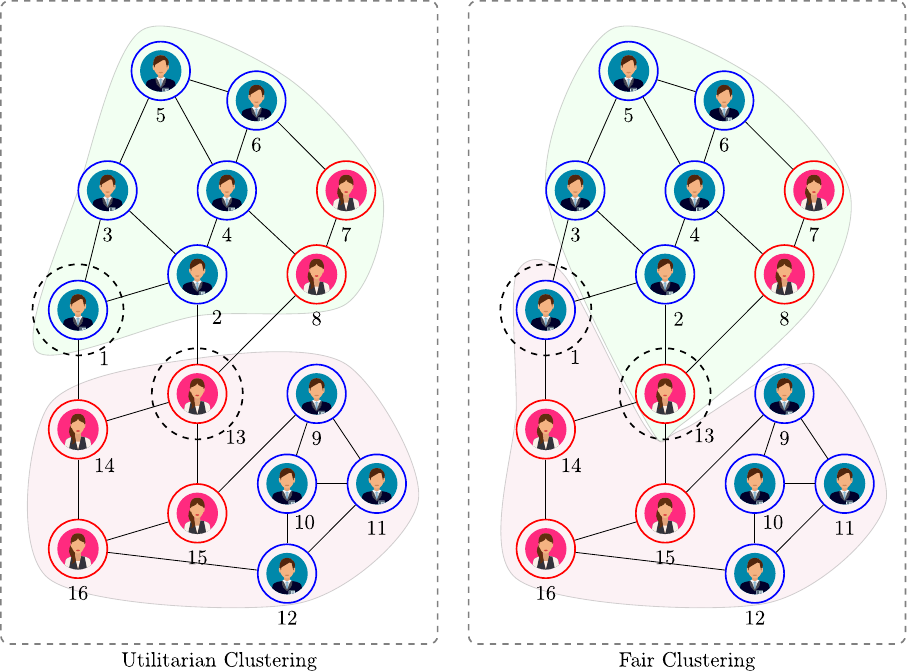}
    \caption{
     Fair clustering of a 16-node graph (10 \textbf{M}ale, 6 \textbf{F}emale) into two equal-sized clusters. Left: Utilitarian clustering yields a structure-driven partition with a 6M:2F distribution for green and 4M:4F for lavender cluster, resulting in gender imbalance. Right: Fair clustering achieves a balanced 5M:3F distribution in both clusters by swapping memberships of nodes 1 and 13.}
    \label{fig:teaser_fig}
\end{figure}

Two lines of research relate to our work. \textbf{(i) Fair clustering for i.i.d.\ data}: develop demographic-parity constraints for $k$-center/median/means via fairlets and related schemes~\cite{DBLP:conf/nips/Chierichetti0LV17,DBLP:conf/aistats/BateniCEL24}, alongside supervised fair GNNs for labeled settings~\cite{DBLP:conf/kdd/DongKTL21,DBLP:journals/tkdd/ChenRPTWYKDA24,DBLP:conf/aaai/YangLYS24}. These methods assume metric i.i.d.\ data or require labels, making them unsuitable for graph-structured clustering. 
\textbf{(ii) Fair graph clustering}: primarily extends spectral clustering with hard demographic constraints \cite{DBLP:conf/icml/KleindessnerSAM19,DBLP:conf/nips/GuptaD22,wang2023scalable}, but eigen-decompositions and post-hoc rounding (e.g., $k$-means) limit end-to-end control and scalability. Recent matrix-factorization approaches offer more flexibility, though contrastive NMF variants can struggle at larger scales \cite{DBLP:conf/pakdd/GhodsiSN24}.

To address these limitations, we propose \textbf{DFNMF}, an end-to-end deep NMF framework for fair graph clustering. DFNMF integrates balanced fairness constraints—enforcing proportional group representation—directly into the clustering objective, enabling explicit utility–fairness trade-offs via a single parameter \(\lambda\) without post-processing. Deep tri-factorization captures multi-level network structure while maintaining scalability through sparse implementations. To our knowledge, this is the first deep hierarchical NMF model with integrated fairness designed specifically for graph clustering. Our contributions are:
\begin{itemize}
  \item \textbf{End-to-end fairness integration:} Joint representation learning and fairness enforcement without multi-stage pipelines or post-processing.
  \item \textbf{Tunable trade-off control:} Single parameter \(\lambda\) enables precise utility–fairness balance adjustment.
  \item \textbf{Deep hierarchical architecture:} Tri-factorization robust to graph sparsity and group-size heterogeneity.
  \item \textbf{Scalable sparse implementation:} CSR-based operations with efficient alternating updates.
  \item \textbf{Comprehensive evaluation:} Synthetic and real networks with statistical analysis, Pareto studies, and ablations.
\end{itemize}

The remainder of this paper is structured as follows: Section~\ref{sec: rel_works} reviews the current literature and highlights the existing gaps. Section~\ref{sec: background} formulates the problem and provides necessary preliminaries, including theoretical foundations of NMF. Section~\ref{sec: proposed} details our proposed DFNMF model. Section~\ref{sec:experiments} presents the experimental results. Finally, Section~\ref{sec:conclusion} concludes the paper and points out future research directions.

\section{Related Works}\label{sec: rel_works}
Fairness-aware graph clustering aims to mitigate bias propagation by enforcing demographic constraints during network partitioning. While graph fairness has been studied extensively~\cite{DBLP:journals/tkde/DongMWCL23, DBLP:conf/ijcai/JuY0L00024, DBLP:conf/ijcai/WangHW0SOC0Y21}, relatively few works specifically tackle fairness in unsupervised graph clustering. The literature can be broadly divided into (1) methods primarily targeting independent and identically distributed (iid) data, and (2) those tailored specifically for graph-structured data.

\subsection{Fair Clustering for iid Data}
Fairness-aware clustering methods initially emerged for independent and identically distributed (iid) tabular data. Chierichetti et al.~\cite{DBLP:conf/nips/Chierichetti0LV17} introduced ``fairlets'' to enforce demographic parity in $k$-center and $k$-median clustering. Subsequent works extended fairness constraints to $k$-means and related objectives~\cite{DBLP:conf/aistats/BateniCEL24}. Some approaches construct similarity graphs from iid data to enable graph-based fairness reasoning~\cite{DBLP:conf/ewaf/GhodsiN23, DBLP:journals/corr/abs-2206-11436}. However, these methods rely on strong geometric or metric assumptions and cannot natively handle the complex relational dependencies inherent in real-world networks~\cite{DBLP:journals/tkde/DongMWCL23, 10992538/esmaili}.

While these methods address unsupervised fairness, they are fundamentally limited for graph clustering applications. Their experimental evaluations typically use simplified datasets that inadequately reflect the structural complexities of networked data, where connections themselves carry semantic meaning beyond mere similarity.

\subsection{Fair Graph Clustering}
Recent studies directly extend spectral clustering to incorporate fairness constraints for graph data. Kleindessner et al.~\cite{DBLP:conf/icml/KleindessnerSAM19} introduced demographic fairness constraints into spectral decompositions. Subsequent works improved computational efficiency using relaxation techniques and augmented Lagrangian methods \cite{DBLP:conf/nips/GuptaD22,wang2023scalable,DBLP:conf/ecai/Li0M23}. However, these spectral methods face fundamental limitations: rigid fairness constraints, non-end-to-end clustering requiring post-processing (e.g., $k$-means), discretization approximations, and limited interpretability.

Graph embedding-based alternatives such as the method by Dong et al.~\cite{DBLP:conf/kdd/DongKTL21} apply individual fairness through rank-based alignment, yet their embedding-driven strategy risks propagating biases from sensitive attributes. Ghodsi et al.~\cite{DBLP:conf/pakdd/GhodsiSN24} proposed asymmetric NMF with contrastive fairness regularization, offering improved interpretability but limited scalability due to computationally intensive pairwise contrastive terms.

Supervised graph methods have also been explored for fairness-aware tasks. Ranking-based fairness models~\cite{DBLP:conf/kdd/DongKTL21}, fair GNN architectures~\cite{DBLP:journals/tkdd/ChenRPTWYKDA24}, and sensitive attribute neutralization strategies~\cite{DBLP:conf/aaai/YangLYS24} show promise in supervised settings. However, these approaches require labeled data and may suffer from proxy-bias propagation through learned embeddings, limiting their applicability to unsupervised community discovery tasks.


GNN-based clustering methods like DMoN~\cite{DBLP:journals/jmlr/TsitsulinPPM23} provide scalable alternatives but lack integrated fairness constraints. While effective for supervised tasks, adapting such methods to incorporate fairness remains an open challenge.

Table~\ref{table:taxonomy} summarizes existing methods across five key criteria. Most approaches either face scalability limitations, enforce rigid fairness constraints, or provide limited interpretability. Existing approaches trade off flexibility (softness), end-to-end optimization, and scalability—gaps our DFNMF addresses.

Clearly, existing methods either face scalability constraints, rigid fairness enforcement, or limited interpretability.
\setlength{\tabcolsep}{2pt}
\begin{table}[!ht]
    \centering
    \resizebox{\linewidth}{!}{ 
    \begin{tabular}{l c c c c c c c}
        \toprule
        \textbf{Method}  & 
        \textbf{Soft?} &
        \textbf{End-to-End} &
        \textbf{Complexity} & \textbf{Scalability} & \textbf{Interpret.} \\
        \midrule






        


        \cite{DBLP:journals/jmlr/TsitsulinPPM23} DMoN    & \textcolor{Green}{\CheckmarkBold} & \textcolor{Green}{\CheckmarkBold} &     $     O(nk^2 + |E|)$ & Large \textcolor{Green}\boldup & Low \textcolor{Red}{\bolddown}\\

        \cite{DBLP:conf/icml/KleindessnerSAM19} FSC   & \textcolor{Red}{\XSolidBrush}& \textcolor{Red}{\XSolidBrush} &   $    O(n^{3})$ & Small \textcolor{Red}{\bolddown} & Low \textcolor{Red}{\bolddown}\\
        
        \cite{wang2023scalable} sFSC    &  \textcolor{Red}{\XSolidBrush} &  \textcolor{Red}{\XSolidBrush} &       $    O(|E|+n(h^2+k^2))$ & Medium & Low\textcolor{Red}{\bolddown}\\

        \cite{DBLP:conf/nips/GuptaD22} iFSC   &  \textcolor{Red}{\XSolidBrush} & \textcolor{Red}{\XSolidBrush}	 &  $    O(n^{3})$ & Small \textcolor{Red}{\bolddown} & Low \textcolor{Red}{\bolddown}\\
        
        \cite{DBLP:conf/kdd/DongKTL21} GNN-FSC  &   \textcolor{Red}{\XSolidBrush} & \textcolor{Red}{\XSolidBrush}	 &  $    O(n^3)$ & Small \textcolor{Red}{\bolddown} & Low \textcolor{Red}{\bolddown} \\

        \cite{DBLP:conf/ecai/Li0M23} FNM-SC   &  \textcolor{Red}{\XSolidBrush} & \textcolor{Red}{\XSolidBrush}	 &  $    O(n^{4.5}k^{4.5}|E|)$ & Small \textcolor{Red}{\bolddown} & Low \textcolor{Red}{\bolddown} \\
        
        \cite{DBLP:conf/pakdd/GhodsiSN24} iFNMTF   & \textcolor{Green}{\CheckmarkBold} &  \textcolor{Green}{\CheckmarkBold}	 &     $    O(|E|k + nk^2 + k^3)$ & Medium & High \textcolor{Green}\boldup\\

        \midrule

        DFNMF (ours)  & \textcolor{Green}{\CheckmarkBold} &   \textcolor{Green}{\CheckmarkBold} &  $    O(Tpk|E|)$ & Large \textcolor{Green}\boldup & High \textcolor{Green}\boldup\\
        
        \bottomrule
    \end{tabular}
    }
    \caption{Representative methods for fair graph clustering. \emph{Soft?} indicates soft (regularized) vs.\ hard constraints; \emph{End-to-End} denotes no post-hoc rounding; complexity is measured in graph size $n$, edges $|E|$, clusters $k$, and iterations $T$, and \emph{Interpret.} refers to parts-based transparency of assignments/factors.}

    \label{table:taxonomy}
\end{table}


\section{Problem Formulation \& Preliminaries} \label{sec: background}



\subsection{Problem Definition}
Consider an undirected graph $\mathcal{G} = (V, E)$ with adjacency matrix $\bm{A}\in\mathbb{R}^{n\times n}$ encoding edge connections, where $a_{ij}>0$ indicates a positive edge between nodes $v_i, v_j\in V$, and $a_{ij}=0$ otherwise (assuming no self-loops, thus $a_{ii}=0$). Suppose the node set $V$ is partitioned into $m$ disjoint demographic groups based on sensitive attributes, such that $V = \dot\cup_{s\in[m]} V_s$.
Our goal is to find a clustering $C = C_1 \dot\cup \ldots \dot\cup C_k$ into $k$ clusters that maximizes structural cohesion while satisfying demographic balance (Definition~\ref{def:balance}), ensuring fair representation.

\subsection{Demographic Balance}

The \emph{balance} criterion, originally introduced by Chierichetti et al.~\cite{DBLP:conf/nips/Chierichetti0LV17}, ensures that each cluster maintains demographic proportions similar to the global distribution. This concept, first applied to $k$-median, was later adapted to spectral clustering in~\cite{DBLP:conf/icml/KleindessnerSAM19, wang2023scalable}. While initially formulated for iid data, we extend this principle to our context, where demographic balance must be preserved over node partitions, formalized as follows:
\begin{defn}[Generalized Demographic Balance]\label{def:balance}
Given a partitioning of the vertex set $V$ into $k \geq 2$ communities, the clustering is said to be \emph{fair} with respect to a partition into demographic groups $s \in [m]$ if, in each community $C_l$, the proportion of group-$s$ nodes matches its global share:
\begin{equation}\label{eq:org_bal}
\frac{|V_s \cap C_l|}{|C_l|} = \frac{|V_s|}{|V|}, \quad \forall\, s\in\{1, \ldots, m\},\ l\in\{1, \ldots, k\}.
\end{equation}
\end{defn}

Here, $V = \{v_1, v_2, \ldots, v_n\}$ is the set of graph nodes, $|V_s|$ is the number of nodes belonging to demographic group $s$, and $|V| = n$ is the total number of nodes. This constraint operationalizes \emph{statistical parity} for clustering and forms the foundation of the fairness regularizer introduced in Section~\ref{sec: proposed}.

\subsection{\textbf{Nonnegative Matrix Factorization (NMF)}}\label{sec:nmf}
NMF is widely used due to its interpretability, versatility, and adaptability to constraints, performing a low-rank approximation of data \cite{lee1999learning}. Its nonnegativity constraints yield parts-based, sparse representations that are inherently interpretable, making NMF particularly suitable for domains like text mining, image analysis, and bioinformatics \cite{gillis2020nonnegative, DBLP:conf/acml/SeyediGAJM19}. Formally, given data matrix $\bm{X}\in \mathbb{R}^{m\times n}$, NMF seeks basis matrix $\bm{W}\in \mathbb{R}^{m\times k}$ and representation matrix $\bm{H}\in \mathbb{R}^{k\times n}$ by solving:
\begin{align}\label{eq:basicnmf}
\min_{\bm{W},\bm{H}\geq 0}\lVert\bm{X}-\bm{WH}\rVert^2_F,
\end{align}
where $\lVert\bm{A}\rVert_F=\sqrt{\sum_{ij}A_{ij}^2}$ denotes the Frobenius norm.

\subsubsection{\textbf{NMF Tri-Factorization for Graph Clustering}}\label{sec:snmf} 

The general NMF formulation \eqref{eq:basicnmf} is not directly tailored for graph clustering. Hence, Symmetric-NMF~\cite{DBLP:conf/sdm/KuangPD12} was developed specifically for this purpose. Symmetric-NMF factorizes a graph adjacency matrix $\bm{A}\in\mathbb{R}^{n\times n}$ into latent node representations $\bm{H}\in\mathbb{R}^{n\times u}$, encouraging attraction between connected nodes and repulsion otherwise. Formally:
\begin{align}\label{eq:snmf}
\min_{\bm{H}\geq 0}\lVert\bm{A}-\bm{HH}^\top\rVert^2_F.
\end{align}

Extending Symmetric NMF, the NMF Tri-Factorization (NMTF)~\cite{DBLP:conf/ijcai/PeiCS15} introduces cluster-cluster interactions via matrix $\bm{W}$, leading to:
\begin{align}\label{eq:snmtf}
\min_{\bm{H},\bm{W}\geq 0}\lVert\bm{A} -\bm{HWH}^\top\rVert^2_F,
\end{align}
where $\bm{W}$ represents interactions between clusters, improving interpretability and flexibility.

An epistemic comparison of theoretical foundations and differences of NMF-based clustering vs neural network architectures is provided in the 
Appendix~\ref{app:discussion}.



\section{Deep Fair NMF Model} \label{sec: proposed}

This section presents \textbf{DFNMF}, a deep fair tri-factor model for community detection on attributed graphs. A proportional group-balance regularizer is built into the objective, producing soft memberships and \emph{direct} cluster assignments—no post-hoc \(k\)-means. The design (i) performs end-to-end clustering with interpretable nonnegative factors; (ii) couples graph topology with demographic indicators to promote balanced communities; (iii) traverses the utility–fairness trade-off via a single parameter \(\lambda\); and (iv) scales through sparse-friendly alternating updates with shallow pretraining and deep fine-tuning.

\subsection{Soft Balanced Fairness (BF)}\label{sec: bal_constraint}
We encode node–cluster memberships by \(\bm{H}\in\mathbb{R}^{n\times k}\).
For partitioning \(V\) into \(k\) clusters, projecting \(\bm{H}\) by node memberships yields:
\begin{equation}\label{eq:stochastic}
\sum_{i=1}^{n_l} H_{il} = 1,\quad \forall\, l\in\{1,\dots,k\}.
\end{equation}

\noindent This encoding enforces nonnegativity, improving interpretability via parts-based memberships. With column-stochastic normalization~(Eq.~\eqref{eq:stochastic}), each entry \(H_{il}\in[0,1]\) serves as a soft membership weight for node \(v_i\) in cluster \(C_l\).

\begin{defn}[Demographic Group Encoding] \label{def:probabilistic}
    Let the vertex set $V$ be partitioned into $m$ sensitive groups $ V = \dot\cup_{s=1}^m V_s, $
    and define the binary sensitive group indicator for node $v_i$ as
    \[
    g^{(s)}_i = \begin{cases} 
    1, & \text{if } v_i\in V_s,\\[1mm]
    0, & \text{otherwise.}
    \end{cases}
    \]
    We then define the group indicator for each sensitive group as
    \begin{equation}\label{eq:probabilistic}
    \bm{f}^{(s)} = \bm{g}^{(s)} - \frac{|V_s|}{n}\bm{1}_n,\quad s\in \{1,\ldots,m-1\}
    \end{equation}
    where $\bm{1}_n\in\mathbb{R}^n$ is the vector of ones. We obtain the matrix $
    \bm{F} = \big[\bm{f}^{(1)}\; \bm{f}^{(2)}\; \cdots\; \bm{f}^{(m-1)}\big]\in\mathbb{R}^{n\times(m-1)}$ by stacking these $m-1$ vectors as columns. Note: We use $m - 1$ columns to avoid linear dependency, since group proportions sum to one.
\end{defn}

\noindent
The matrix \(\bm{F}\) stacks mean-centered (proportionally centered) group indicators. It will encode demographic balance via the linear condition \(\bm{F}^\top\bm{H}= \bm{0}\) and induce the fairness regularizer \(\|\bm{F}^\top\bm{H}\|_F^2\). This construction is used as Step~1 in Algorithm~\ref{alg:algorithm}.


\begin{defn}[Soft Balanced Fairness] \label{def:fair_nmf}
A partitioning represented by a column-stochastic membership matrix $\bm{H}$ is \emph{fair} with respect to the sensitive groups if, for each cluster $C_l$ and every group $s\in\{1,\ldots,m\}$, the proportion of mass from group $V_s$ in $C_l$ equals the global group proportion. It means that, a clustering $\bm{H}$ is fair if eq.\eqref{eq:org_bal} strictly holds for all clusters and all sensitive groups.
In our soft clustering formulation, this condition is equivalent to the following:
\[
\sum_{i\in V_s} H_{il} = \frac{|V_s|}{n}\quad \text{for all } s\in \{1,\ldots,m-1\}, \ l\in \{1,\ldots,k\}.
\]
Equivalently, stacking the deviations for all groups and clusters, the fairness condition can be formulated as $\bm{F}^\top\bm{H} = \bm{0}$.
\end{defn}

\begin{lem}[Equivalence to Demographic Balance]\label{lem:fairness}
Let $\bm{H}$ be a nonnegative, column-stochastic matrix. Then $\bm{F}^\top\bm{H} = \bm{0}$ of Definition~\ref{def:fair_nmf}, is equivalent to the fairness condition in Eq.~\eqref{eq:org_bal}.

\begin{proof}
Since each column of $\bm{H}$ is normalized by Eq.~\eqref{eq:stochastic}, for any cluster $l$ we have $
\sum_{i=1}^n H_{il} = 1.$
Then, for each sensitive group $s\in[m-1]$ and cluster $l$, the $(s,l)$th entry of $\bm{F}^\top\bm{H}$ is
\begin{align*}
(\bm{F}^\top\bm{H})_{sl} 
& = \sum_{i=1}^n f^{(s)}_i H_{il}
= \sum_{i\in V_s} H_{il} - \frac{|V_s|}{n}\sum_{i=1}^n H_{il} \\
& = \sum_{i\in V_s} H_{il} - \frac{|V_s|}{n}.
\end{align*}

\noindent Thus, $(\bm{F}^\top\bm{H})_{sl}=0$ if and only if $
\sum_{i\in V_s} H_{il} = \frac{|V_s|}{n},
$
which is exactly equal to the fairness condition in Eq.~\eqref{eq:org_bal}. 
\end{proof}
\end{lem}

\begin{figure*}[ht]
    \centering
    \includegraphics[width=0.9\textwidth]{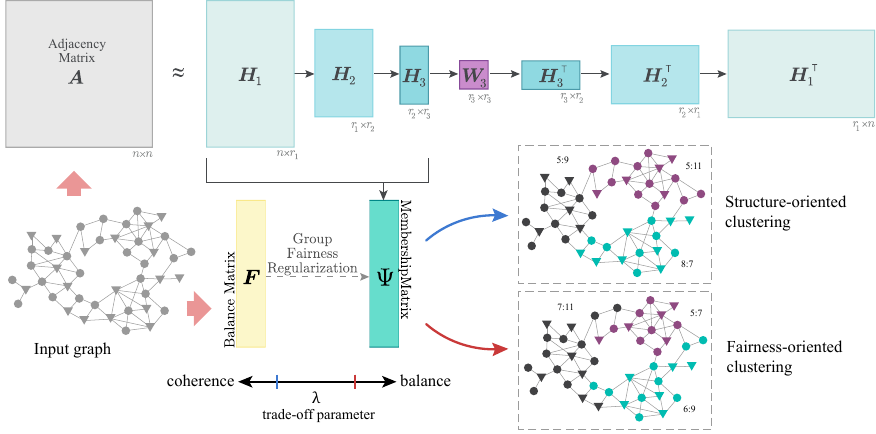}
    \caption{\textbf{DFNMF schematic and example.} A 45-node graph with imbalanced gender distribution of 40\%/60\%(27~\includegraphics[height=1.2ex]{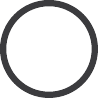}, 18~\includegraphics[height=1.2ex]{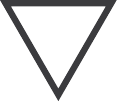}) is factorized through $\bm{H}_1,\bm{H}_2,\bm{H}_3$. Two solutions illustrate the effect of $\lambda$: small $\lambda$ preserves structure but yields imbalance (5:9, 5:11, 8:7); large $\lambda$ improves parity (7:11, 5:7, 6:9), highlighting the utility–fairness trade-off.}

    \label{fig:enter-label}
\end{figure*}

\subsection{Shallow Fair NMF Model}\label{sec: pre-train}
We begin with a shallow Fair NMF (FNMF) that serves as the foundation for our deep hierarchical extension. To encode fairness, we augment the standard tri-factorization with a penalty derived from Lemma~\ref{lem:fairness}, encouraging proportional representation during optimization:
\begin{align}\label{eq:basic}
    \min_{\bm{H},\,\bm{W}\ge 0}\;\;\|\bm{A}-\bm{H}\bm{W}\bm{H}^\top\|_F^2\;+\;\lambda\,\mathcal{R}(\bm{H}),
\end{align}
where \(\bm{H}\in\mathbb{R}_{+}^{n\times k}\) are soft cluster memberships and \(\bm{W}\in\mathbb{R}_{+}^{k\times k}\) captures inter-cluster interactions.

\paragraph{Fairness penalty.}
We enforce demographic balance via the smooth quadratic penalty (i.e. $L_2$ norm): 
\begin{equation}\label{eq:fair_reg}
\mathcal{R}(\bm{H})=\|\bm{F}^\top\bm{H}\|_F^2
=\sum_{s=1}^{m-1}\sum_{l=1}^{k}\big(\bm{f}^{(s)\top}\bm{h}_l\big)^2,
\end{equation}
where \(\bm{f}^{(s)}\) is column \(s\) of \(\bm{F}\) and \(\bm{h}_l\) is column \(l\) of \(\bm{H}\). By Lemma~\ref{lem:fairness}, \(\mathcal{R}(\bm{H})=0\) iff the balance condition holds; decreasing \(\mathcal{R}\) drives cluster compositions toward global group proportions.

Substituting \eqref{eq:fair_reg} into \eqref{eq:basic} yields our shallow FNMF objective:
\begin{equation}\label{eq:FairSNMF}
\min_{\bm{H},\,\bm{W}\ge 0}\;
\underbrace{\|\bm{A}-\bm{H}\bm{W}\bm{H}^\top\|_F^2}_{\text{utility term}}
\;+\;
\underbrace{\lambda\,\|\bm{F}^\top\bm{H}\|_F^2}_{\text{fairness term}},
\end{equation}
with \(\lambda\) controlling the utility–fairness trade-off.

\subsection{Deep Fair NMF (DFNMF) Model}\label{sec: fine-tune}
Shallow tri-factorization is transparent and efficient, but can miss multi-level structure on large graphs. We therefore adopt a deep tri-factorization inspired by~\cite{DBLP:journals/pr/HajiveisehST24} in which successive nonnegative layers capture increasingly coarse communities. Let
\begin{equation}\label{eq:psi-def}
\bm{\Psi} \;=\; \prod_{i=1}^{p}\bm{H}_i, \quad \bm{H}_i\!\in\!\mathbb{R}_+^{r_{i-1}\times r_i}
\end{equation}
\noindent where $r_0{=}n \ge \cdots \ge r_p{=}k$.
and \(\bm{W}_p\in\mathbb{R}_+^{k\times k}\) encode final inter-/intra-cluster interactions. The graph is reconstructed as
\begin{equation}\label{eq:recon}
\bm{A} \;\approx\; \bm{\Psi}\,\bm{W}_p\,\bm{\Psi}^\top
\end{equation}
\noindent yielding direct soft memberships (columns of \(\bm{\Psi}\)) without any post-hoc clustering.

\paragraph{Unified objective.}
We integrate the balance penalty from Eq.~\eqref{eq:fair_reg} at the final layer:
\begin{equation}\label{eq:final}
\min_{\{\bm{H}_i\}_{i=1}^p\ge0,\;\bm{W}_p\ge 0}\mathcal{L}=\;
\underbrace{\big\|\bm{A}-\bm{\Psi}\bm{W}_p\bm{\Psi}^\top\big\|_F^2}_{\text{utility}}
\;+\;
\underbrace{\lambda\,\big\|\bm{F}^\top\bm{\Psi}\big\|_F^2}_{\text{fairness}},
\end{equation}
The regularizer acts on \(\bm{\Psi}\), i.e., the final memberships, and \(\lambda\) tunes the utility–fairness trade-off.

\paragraph{Pipeline illustration.}
Figure~\ref{fig:enter-label} depicts the hierarchy (\(\bm{H}_1,\bm{H}_2,\bm{H}_3\)) and a 45-node example with two sensitive groups (triangles/circles). Small \(\lambda\) preserves intrinsic structure but yields imbalanced clusters (e.g., \(5{:}9\), \(5{:}11\), \(8{:}7\)); large \(\lambda\) adjusts memberships toward parity (e.g., \(7{:}11\), \(5{:}7\), \(6{:}9\)). This demonstrates controllable movement along the utility–fairness spectrum by tuning \(\lambda\).

\subsection{Optimization}\label{sec: opt}
We solve the DFNMF objective in Eq.~\eqref{eq:final} with alternating minimization~\cite{DBLP:conf/icml/TrigeorgisBZS14}, updating each factor while holding the others fixed. As the problem is non-convex, this procedure converges to a local optimum but does not guarantee global optimality.

\paragraph{Pretraining (Algorithm~\ref{alg:algorithm}, lines 2–6).}
To accelerate convergence, we initialize each layer via sequential NMTF: first factorize $\bm{A}\approx \bm{H}_1\bm{W}_1\bm{H}_1^\top$, then recursively factorize $\bm{W}_{i-1}\approx \bm{H}_i\bm{W}_i\bm{H}_i^\top$ for $i=2,\dots,p$. This provides warm starts for $\{\bm{H}_i\}$ and $\bm{W}_p$, reducing wall time in practice~\cite{DBLP:journals/csr/HandschutterGS21}.

\paragraph{Fine-tuning (Algorithm~\ref{alg:algorithm}, lines 7–16).}
With $\bm{\Psi}$ defined as in Eq.~\eqref{eq:psi-def}, we alternate updates for the membership blocks $\{\bm{H}_i\}$ and the interaction matrix $\bm{W}_p$.

\subsubsection{Update rule for membership blocks \texorpdfstring{$\bm{H}_i$}{\textbf{H_i}}}
Fix all variables except $\bm{H}_i$. Using the block products
\[
\bm{\Psi}_i \;=\; \bm{H}_1\cdots \bm{H}_{i-1},\qquad
\bm{\Phi}_i \;=\; \bm{H}_{i+1}\cdots \bm{H}_{p},
\]
(with $\bm{\Psi}_1=\bm{I}$ and $\bm{\Phi}_p=\bm{I}$), the subproblem is
\begin{align}\label{eq:Hi}
\min_{\bm{H}_i\ge 0}\;
\mathcal{L}(\bm{H}_i)
&= \big\|\bm{A} - \bm{\Psi}_i \bm{H}_i \bm{\Phi}_i \bm{W}_p \bm{\Phi}_i^\top \bm{H}_i^\top \bm{\Psi}_i^\top\big\|_F^2
\\[-2mm]
&\hspace{8mm}
+ \lambda\,\big\|\bm{F}^\top \bm{\Psi}_i \bm{H}_i \bm{\Phi}_i\big\|_F^2.\nonumber
\end{align}
Following a standard Lee–Seung style multiplicative update derived~\cite{lee1999learning} via KKT conditions (derivation omitted for brevity), we obtain
\begin{equation}\label{eq:uHi}
\bm{H}_i \;\leftarrow\; \bm{H}_i \odot \Bigg(\frac{\mathcal{N}_i}{\mathcal{D}_i}\Bigg)^{\!1/4},
\end{equation}
where we introduce compact shorthands
\begin{align*}
\mathcal{N}_i &= \bm{\Psi}_i^\top\!\left(\bm{A}^\top \bm{\Psi}\bm{W}_p \;+\; \bm{A}\bm{\Psi}\bm{W}_p^\top \;+\; \lambda\,\big[\bm{F}\bm{F}^\top \bm{\Psi}\big]^-\right)\!\bm{\Phi}_i^\top,\\
\mathcal{D}_i &= \bm{\Psi}_i^\top\!\left(\bm{\Psi}\bm{W}_p^\top\bm{\Psi}^\top\bm{\Psi}\bm{W}_p \;+\; \bm{\Psi}\bm{W}_p \bm{\Psi}^\top\bm{\Psi}\bm{W}_p^\top \right.\\[-1mm]
&\hspace{18mm}\left. +\; \lambda\,\big[\bm{F}\bm{F}^\top \bm{\Psi}\big]^+\right)\!\bm{\Phi}_i^\top,
\end{align*}
and $\bm{B}^+=\max(\bm{B},0)$, $\bm{B}^-=-\min(\bm{B},0)$ so that $\bm{B}=\bm{B}^+-\bm{B}^-$. Update \eqref{eq:uHi} is used in Algorithm~\ref{alg:algorithm}, line~13. The detailed derivation is discussed in the Appendix~\ref{app:der}.

\begin{algorithm}[tb]
	\caption{[Balanced] Deep Fair NMF (DFNMF)}
	\label{alg:algorithm}
	\begin{flushleft}
		\textbf{Input}: The adjacency matrix of graph $\mathcal{G}, \bm{A}$; layer size of each layer, $r_i$; fairness regularization parameter $\lambda$;\\
		\textbf{Output}: $\bm{W}_i\ (1\leq i< p), \bm{H}_i\ (1\leq i< p)$, and the membership matrix $\bm{\Psi};$ 
	\end{flushleft}
	\begin{algorithmic}[1] 
		\STATE Constructing the balance fairness matrix $\bm{F}$ according to Definition~\ref{def:probabilistic};
		\STATE \textbf{$\triangleright$ Pre-training process:} 
		\STATE $\bm{W}_1, \bm{H}_1 \leftarrow $NMTF$(\bm{A},r_1)$; 
		\FOR{$i = 2\ \textbf{to}\ p$} 
		\STATE $\bm{W}_i, \bm{H}_i \leftarrow $NMTF$(\bm{W}_{i-1},r_i)$; 
		\ENDFOR
		\STATE \textbf{$\triangleright$ Fine-tuning process:} following Section~\ref{sec: fine-tune}
        \STATE $\bm{W}_i, \bm{H}_i \leftarrow $DFNMF$(\bm{\Psi}_i,p)$; deep hierarchical learning
		\WHILE{convergence not reached}
		\FOR{$i = 1\ \textbf{to}\ p$}
		\STATE $\bm{\Psi}_{i-1}\leftarrow \prod_{\tau=1}^{i-1} \bm{H}_\tau (\bm{\Psi}_{0}\leftarrow \textbf{I})$;
		\STATE $\bm{\Phi}_{i+1}\leftarrow \prod_{\tau=i+1}^{p} \bm{H}_\tau (\bm{\Phi}_{p+1}\leftarrow \textbf{I})$;
		\STATE Update clustering matrix $\bm{H}_i$ using~\eqref{eq:uHi};
		\STATE $\bm{\Psi}_{i}\leftarrow \bm{\Psi}_{i-1}\bm{H}_i$;
            \STATE Update interaction matrix $\bm{W}_p$ using~\eqref{eq:uWp};
		\ENDFOR
		\ENDWHILE 
	\end{algorithmic}
\end{algorithm}

\subsubsection{Update rule for interaction matrix \texorpdfstring{$\bm{W}_p$}{\textbf{W_p}}}
By fixing $\{\bm{H}_i\}_{i=1}^p$, we can update $\bm{W}_p$ by solving
\begin{align}
	\min_{\bm{W}_p}\mathcal{L}(\bm{W}_p)=&\lVert \bm{A}- \bm{\Psi} \bm{W}_p\bm{\Psi}^{\top}\rVert_F^2, \text\ {\text{s.t.}\quad}\bm{W}_p\geq 0,
\end{align}

\noindent Let $\bar{\bm{A}}=\bm{\Psi}^\top\bm{A}\bm{\Psi}$ and $\bm{S}=\bm{\Psi}^\top\bm{\Psi}$. The multiplicative update becomes
\begin{equation}\label{eq:uWp}
\bm{W}_p \;\leftarrow\; \bm{W}_p \odot \frac{\bar{\bm{A}}}{\bm{S}\,\bm{W}_p\,\bm{S}},
\end{equation}
which corresponds to Algorithm~\ref{alg:algorithm}, line~15.

\paragraph{Remarks.}
(i) Eqs.~\eqref{eq:uHi}–\eqref{eq:uWp} preserve nonnegativity and monotonically decrease the objective under the usual assumptions for multiplicative updates. (ii) The shorthands $\mathcal{N}_i,\mathcal{D}_i$ compactly collect terms as auxiliary variables improving readability. 

\subsection{Computational Complexity}\label{sec:complexity}
Our method involves iterative updates of the factor matrices \(\bm{W}_p\)~\eqref{eq:uWp} and \(\bm{H}_i\)~\eqref{eq:uHi} until convergence. The update of \(\bm{W}_p\) has a complexity of \(\mathcal{O}(n^2k + nk^2 + k^3)\), while the update of each \(\bm{H}_i\), which incorporates the fairness regularization, has a complexity of \(\mathcal{O}(n^2k + nk^2 + n(m{-}1)k + nr_ik + r_i(m{-}1)k)\). In practice, \(k\), \(m\), and \(r_i\) are typically small, making the overall cost dominated by operations involving \(n\). To further improve efficiency, our implementation uses sparse CSR representations of \(\bm{A}\), reducing the dominant matrix operations to \(\mathcal{O}(|E|k)\), where \(|E|\) is the number of non-zero entries in \(\bm{A}\) (i.e. number of edges). This yields an overall per-iteration complexity of \(\mathcal{O}(p|E|k)\). Optional block coordinate descent (not used in our experiments) can further reduce memory from $\mathcal{O}(n^2)$ to $\mathcal{O}(bk)$.

\section{Experiments} \label{sec:experiments}
This section presents experimental evaluations of our proposed model compared to baselines on both real-world and synthetic datasets. Results are assessed based on the utility and fairness of partitioning under various conditions.

\subsection{Experimental Setup} 
All experiments were conducted on a NVIDIA A3090 GPU with 24 GB of memory. Sparse graph operations were implemented using the \texttt{scipy.sparse} classes of \texttt{csr\_matrix} and \texttt{coo\_matrix}, and their counterparts in \texttt{torch.sparse}. Optimized sparse implementations are detailed in code repository. We use a 500 epochs max, and $p=4$ layers with 64 components (unless otherwise stated). The final layer is always projected to $k$, the number of communities (i.e., clusters) for DFNMF. The depth ($p=4$) offers sufficient expressiveness while avoiding over-smoothing in deeper architectures~\cite{ DBLP:journals/csr/HandschutterGS21}; deeper variants showed no empirical improvement. All results and ablations (Appendix~2~\ref{app:ablation}) average over 10 seeds; the $\lambda^\star$ selection rule was cross-checked via a linear scalarization to confirm robustness (see Section~\ref{sec:ablation}).

\subsubsection{\textbf{Unsupervised Setting}} 
As we address the unsupervised task of clustering, there is no train-validation-test split. Instead, the entire dataset is used as input during each model run. Importantly, no ground-truth labels are accessed during training or inference. Label-based metrics (e.g., ARI or accuracy) are used only \emph{post hoc} for evaluation assessment.

\subsubsection{\textbf{Parameter Selection}} 

We sweep $\lambda$ on a logarithmic grid $\{10^{-3},\ldots,10^{3}\}$ and form the utility–fairness curve
$(\tilde Q(\lambda),\tilde B(\lambda))$ after min–max scaling of $Q$ and $\bar B$ to $[0,1]$ (per dataset and $k$).
We first retain the \emph{Pareto front} $\Lambda_{\!\mathrm{P}}$ (undominated points), then select
\[
  \lambda^\star \;=\; \arg\min_{\lambda\in\Lambda_{\!\mathrm{P}}}
  \bigl\|\,(\tilde Q(\lambda),\tilde B(\lambda)) - (1,1)\,\bigr\|_2,
\]
with a tie–breaker favoring smaller $|\tilde Q(\lambda)-\tilde B(\lambda)|$ (closer to the identity guide as shown in Section~\ref{sec:lam}).
This setting obtains a robust and reproducible $\lambda^\star$ scheme, which are reported together with their bracket $\{\lambda^\star/10,\,10\lambda^\star\}$ in the appendix~B. All the experimental results throughout the paper are reported according to our optimal $\lambda^\star$ for each dataset.

\subsubsection{\textbf{Baselines}}
We compare DFNMF with seven 
baselines introduced previously in Section~\ref{sec: rel_works}. They comprise vanilla models of spectral clustering (SC)~\cite{DBLP:journals/sac/Luxburg07}, tri-factor NMF (NMTF)~\cite{DBLP:conf/sdm/KuangPD12}, GNN-based model (DMoN)~\cite{DBLP:journals/jmlr/TsitsulinPPM23}, and fairness-aware models: fair, scalable, individual spectral models (FairSC, sFSC, iFSC)~\cite{DBLP:conf/icml/KleindessnerSAM19, wang2023scalable, DBLP:conf/nips/GuptaD22}, and individually fair NMTF (iFNMTF)~\cite{DBLP:conf/pakdd/GhodsiSN24}. All hyperparameters are tuned according to the original papers. 

\subsubsection{\textbf{Evaluation Measures}}
We assess clustering \emph{utility} on labeled datasets using \textbf{accuracy (ACC)} and \textbf{adjusted Rand index (ARI)}, and on unlabeled datasets using Newman's \textbf{modularity (Q)}~\cite{DBLP:journals/csur/ChakrabortyDMG17}. Accuracy calculates the proportion of nodes correctly assigned to their respective clusters. ARI is a pairwise agreement measure between predicted and ground-truth clusterings, adjusted for chance, and ranges from $-1$ (random clustering), through $0$ (chance-level agreement), to $1$ (perfect clustering). Modularity quantifies the strength of division of a graph into clusters by comparing the observed intra-cluster connectivity to that expected under a random null model, also ranging from $-1$ to $1$. 
For \emph{fairness}, we report \textbf{average balance ($\bar{B}$)}~\cite{wang2023scalable, DBLP:conf/icml/KleindessnerSAM19} and \textbf{statistical parity deviation ($\Delta_{SP}$)}~\cite{DBLP:conf/icse/VermaR18}. $\bar{B}$ computes the mean minimum group proportion across clusters, where higher values in $[0, 1]$ indicate fairer clusters. $\Delta_{SP}$ measures deviation from global group proportions within each cluster:
\begin{equation}\label{eq:spd}
\Delta_{SP} = \frac{1}{k}\sum_{l=1}^{k}\sum_{s=1}^{m}\left|\frac{|V_s\cap C_l|}{|C_l|}-\frac{|V_s|}{|V|}\right|,
\end{equation}
where lower values (approaching $0$) imply better demographic parity. 
All metrics are averaged over 10 runs. ACC and $\bar{B}$ lie in $[0, 1]$, while ARI, $Q$, and $\Delta_{SP}$ range in $[-1, 1]$.

\subsection{Datasets} \label{sec: datasets}
We evaluate DFNMF on 11 networks (8 real-world, 3 synthetic) with diverse structural setups, group imbalances, and homophily levels. Dataset characteristics, including number of nodes ($|V|$), edges ($|E|$), sensitive groups, number of clusters ($k$), edge density, and homophily, are summarized in Table~\ref{table:dataset}. 

Homophily quantifies the tendency of nodes to connect within the same sensitive group, indicating inherent network bias. Thus, datasets with higher homophily tend to exacerbate the impact of demographic imbalances on fairness.

\begin{table}[!ht]
    \centering
    \footnotesize
    \setlength{\tabcolsep}{4pt}
    \caption{
    Summary statistics for datasets used in experiments.}
    \label{table:dataset}
    \begin{tabular}{llllccc}
        \toprule
        \textbf{Network} & \textbf{$|V|$} & \textbf{$|E|$} & \makecell{\textbf{Attribute}\\\textbf{(\# groups)}} & \makecell{\textbf{Edge}\\\textbf{Density}} & \textbf{Homophily} & \textbf{$k$}\\
        \midrule
        \multirow{3}{*}{SBM}
        & 2,000 & 267,430 & attr (2) & 0.133 & 0.82 & 5 \\
        & 5,000 & 978,959 & attr (2) & 0.078 & 0.82 & 5 \\
        & 10,000 & 2,603,190 & attr (2) & 0.052 & 0.82 & 5 \\
        \hdashline
        Pokec-n & 67,797 & 882,765 & age (4) & 0.0384 & 0.399 & 2\\
        Pokec-z & 66,570 & 729,129 & age (4) & 0.0329 & 0.360 & 2\\
        NBA & 403 & 8,285 & ethnicity (2) & 0.102 & 0.72 & 2\\
        \hdashline
        Diaries & 120 & 348 & gender (2) & 0.048 & 0.61 & --\\
        Friendship & 134 & 406 & gender (2) & 0.049 & 0.60 & --\\
        Facebook & 156 & 1,437 & gender (2) & 0.120 & 0.57 & --\\
        DrugNet & 293 & 284 & ethnicity (3) & 0.014 & 0.88 & --\\
        LastFM & 7,624 & 27,806 & country (6) & 0.001 & 0.92 & --\\
        \bottomrule
    \end{tabular}
\end{table}
\begin{figure*}[!t]
  \centering
  \includegraphics[width=\linewidth]{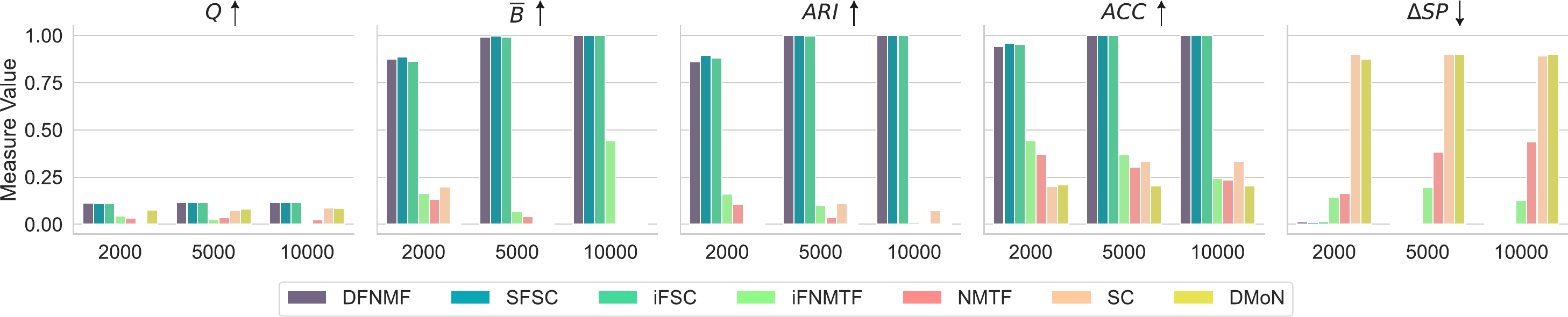}
  \caption{SBM networks with varying node sizes: comparison of clustering and fairness metrics. Arrows ($\uparrow$/$\downarrow$) indicate whether higher/lower is better.}
  \label{fig:SBM_group_bar}
\end{figure*}
\subsubsection{\textbf{Synthetic Networks}} are generated using an extended Stochastic Block Model (\textbf{SBM}) 
following~\cite{DBLP:conf/icml/KleindessnerSAM19, wang2023scalable, DBLP:conf/nips/GuptaD22}, with explicitly controlled group memberships and cluster assignments. Each node set $V$=[$n$] is partitioned into \(m\) groups \(V=V_1\dot\cup\ldots\dot\cup V_m\) and \(k\) clusters \(V=C_1\dot\cup\ldots\dot\cup C_k\), ensuring proportional representation of each sensitive group within clusters (fair clustering by construction). 
\subsubsection{\textbf{Real-World Networks}} 
Our collection of real datasets represents diverse social and interaction graphs with varying demographics and biases~\cite{DBLP:conf/wsdm/DaiW21, DBLP:conf/cikm/RozemberczkiS20, weeks2002social, mastrandrea2015contact}: small-scale datasets (\textit{Facebook}, \textit{Friendship}, \textit{Diaries}) exhibit moderate gender imbalance (around 60\% majority groups), and medium-sized datasets like \textit{DrugNet} and \textit{NBA} show stronger imbalances (majority ethnic groups above 70\%). Larger datasets such as \textit{Pokec-n} and \textit{Pokec-z} present considerable imbalance (majority age group over 70\%) with moderate homophily (0.36--0.40), while highly sparse and homophilous graphs like \textit{LastFM} (country, homophily = 0.92) and \textit{DrugNet} (ethnicity, homophily = 0.88) pose greater fairness challenges. These variations allow comprehensive exploration of fairness-utility trade-offs.

\subsection{Results on Synthetic Datasets}
To benchmark the baselines under controlled conditions, we conducted experiments on three variants of the SBM networks with 2,000, 5,000, and 10,000 nodes. Results are depicted in Figure~\ref{fig:SBM_group_bar}. SFSC and FSC baselines report identical results; thus, only SFSC is visualized to conserve space. The results indicate that DFNMF, SFSC, and iFSC consistently outperform other baselines, including vanilla models (SC, NMTF, DMoN) and the iFNMTF by reporting the highest $Q$, $\overline{B}$, $ARI$, $ACC$, and lowest $\Delta SP$. These findings strongly support that the three models have an advantage in trading-off clustering utility and fairness under varying degrees of structural complexity and group imbalance on SBM networks.

\subsection{Results on Real-world Datasets}
We further evaluate DFNMF and baseline models on a variety of real-world networks, with diverse structural and fairness challenges. These include datasets with strong group imbalance (e.g., Pokec, NBA), as well as sparse, highly homophilous graphs such as \emph{LastFM} and \emph{DrugNet}, as characterized in Section~\ref{sec: datasets} and Table~\ref{table:dataset}. Given the comparable performance of spectral methods on synthetic data, we focus here on their limitations in more complex, real-world scenarios.

\subsubsection{\textbf{UnLabeled Datasets}}

Table~\ref{tab:fair_graph_balance_modularity} presents results on datasets without ground-truth labels, evaluated using average balance ($\overline{B}$) and modularity ($Q$) under two cluster settings ($C$=5 and $C$=10).
On \emph{DrugNet} and \emph{LastFM}, DFNMF shows a clear advantage, achieving both high modularity and fairness. In contrast, spectral baselines (SFSC, FSC, iFSC) perform poorly due to their rigid fairness constraints, which often fail to adapt to highly sparse and biased network structures. On smaller networks such as \emph{Facebook}, \emph{Friendship}, and \emph{Diaries}, DFNMF maintains top fairness performance and ranks second in modularity—slightly behind SFSC or iFSC—demonstrating strong overall adaptability even in simpler settings.

\begin{table}[!t]
\centering
\footnotesize
\caption{Performance on unLabeled datasets for k=5 and k=10 number of clusters.}
\resizebox{\columnwidth}{!}{%
\begin{tabular}{ll|cc|cc|cc|cc|cc}
\toprule
\multirow{2}{*}{\textbf{Model}} & \multirow{2}{*}{\textbf{}}
  & \multicolumn{2}{c}{\textbf{DrugNet}}
  & \multicolumn{2}{c}{\textbf{Diaries}}
  & \multicolumn{2}{c}{\textbf{Facebook}}
  & \multicolumn{2}{c}{\textbf{Friendship}}
  & \multicolumn{2}{c}{\textbf{LastFM}} \\
\cmidrule(lr){3-4} \cmidrule(lr){5-6} \cmidrule(lr){7-8} \cmidrule(lr){9-10} \cmidrule(lr){11-12}
  & & k=5 & k=10 & k=5 & k=10 & k=5 & k=10 & k=5 & k=10 & k=5 & k=10 \\
\midrule
\multirow{2}{*}{DMoN}
  & \(\overline{B}\)
  & 0.00 & 0.00   & 0.263 & 0.034   & 0.267 & 0.194   & 0.182 & 0.040   & 0.00 & 0.00 \\
  & \(Q\)
  & 0.326 & 0.324 & 0.145 & 0.165 & 0.047 & 0.034 & 0.129 & 0.121 & 0.526 & 0.360 \\
\midrule
\multirow{2}{*}{SC}
  & \(\overline{B}\)
  & 0.030 & 0.020   & 0.681 & 0.467   & 0.295 & 0.181   & 0.362 & 0.282   & 0.005 & 0.002 \\
  & \(Q\)
  & \textbf{0.607} & \textbf{0.633}
    & 0.631 & 0.669
    & 0.454 & 0.450
    & 0.575 & 0.602
    & 0.418 & 0.419 \\
\midrule
\multirow{2}{*}{SFSC}
  & \(\overline{B}\)
  & 0.052 & 0.026    & 0.684 & 0.494   & 0.601 & 0.436   & 0.485 & 0.399   & 0.067 & 0.033 \\
  & \(Q\)
  & 27.00 & 52.56
    & \textbf{0.808} & \textbf{0.698}
    & 0.500 & 0.512
    & 0.626 & 0.665
    & 0.034 & 0.040 \\
\midrule
\multirow{2}{*}{FSC}
  & \(\overline{B}\)
  & 0.052 & 0.026    & 0.684 & 0.494   & 0.601 & 0.436   & 0.485 & 0.399   & 0.067 & 0.033 \\
  & \(Q\)
  & 0.270 & 0.525
    & 0.808 & 0.698
    & 0.500 & \textbf{0.512}
    & 0.626 & 0.665
    & 0.034 & 0.040 \\
\midrule
\multirow{2}{*}{iFSC}
  & \(\overline{B}\)
  & 0.055 & 0.037    & \textbf{0.800} & 0.495   & 0.564 & 0.452   & 0.581 & 0.520   & 0.065 & 0.032 \\
  & \(Q\)
  & 0.279 & 0.457
    & 0.657 & 0.697
    & \textbf{0.515} & 0.496
    & 0.624 & \textbf{0.683}
    & 0.007 & 0.013 \\
\midrule
\multirow{2}{*}{iFNMF}
  & \(\overline{B}\)
  & 0.098 & 0.114  & 0.706 & 0.578   & 0.527 & 0.001   & 0.613 & 0.551   & 0.071 & 0.029 \\
  & \(Q\)
  & 0.116 & 0.058
    & 0.238 & 0.129
    & 0.239 & 0.004
    & 0.216 & 0.621
    & 0.254 & 0.225 \\
\midrule
\multirow{2}{*}{DFNMF}
  & \(\overline{B}\)
  & \textbf{0.162} & \textbf{0.141}
    & 0.787 & \textbf{0.707}
    & \textbf{0.767} & \textbf{0.614}
    & \textbf{0.614} & \textbf{0.623}
    & \textbf{0.091} & \textbf{0.084} \\
  & \(Q\)
  & 0.591 & 0.524
    & 0.716 & 0.594
    & 0.503 & 0.432
    & \textbf{0.666} & 0.604
    & \textbf{0.420} & \textbf{0.455} \\
\bottomrule
\end{tabular}
}
\label{tab:fair_graph_balance_modularity}
\end{table}

\begin{table}[!htbp]
\centering
\footnotesize
\captionsetup{font=normalsize}
\caption{Performance (\textbf{mean $\pm$ std}) on Labeled datasets. Arrows ($\uparrow$/$\downarrow$) show if higher/lower is better.}
\setlength{\tabcolsep}{4pt}
\begin{tabularx}{\linewidth}{llccc}
\toprule
\textbf{Metric} & \textbf{Model} & \textbf{Pokec-n} & \textbf{Pokec-z} & \textbf{NBA} \\
\midrule
\multirow{3}{*}{$ARI$ $\uparrow$}
& DFNMF & \textbf{0.0051 $\pm$ 0.001} & \textbf{0.0158 $\pm$ 0.002} & \textbf{0.1203 $\pm$ 0.012} \\
& SFSC  & 0.0009 $\pm$ 0.000 & 0.0009 $\pm$ 0.000 & 0.0825 $\pm$ 0.007 \\
& DMoN  & 0.0024 $\pm$ 0.000 & 0.0022 $\pm$ 0.000 & 0.0741 $\pm$ 0.006 \\
\midrule
\multirow{3}{*}{$Q$ $\uparrow$}
& DFNMF & \textbf{0.2067 $\pm$ 0.005} & \textbf{0.1910 $\pm$ 0.026} & \textbf{0.1344 $\pm$ 0.012} \\
& SFSC  & 0.0001 $\pm$ 0.000 & 0.0001 $\pm$ 0.000 & 0.1089 $\pm$ 0.013 \\
& DMoN  & 0.1801 $\pm$ 0.003 & 0.1625 $\pm$ 0.004 & 0.1162 $\pm$ 0.008 \\
\midrule
\multirow{3}{*}{$ACC$ $\uparrow$}
& DFNMF & 0.6879 $\pm$ 0.052 & 0.7767 $\pm$ 0.096 & \textbf{0.6765 $\pm$ 0.010} \\
& SFSC  & \textbf{0.8476 $\pm$ 0.012} & \textbf{0.8165 $\pm$ 0.042} & 0.6500 $\pm$ 0.004 \\
& DMoN  & 0.5252 $\pm$ 0.035 & 0.5062 $\pm$ 0.039 & 0.5232 $\pm$ 0.026 \\
\midrule
\multirow{3}{*}{$\overline{B}$ $\uparrow$}
& DFNMF & \textbf{0.1844 $\pm$ 0.017} & \textbf{0.2905 $\pm$ 0.025} & \textbf{0.4373 $\pm$ 0.028} \\
& SFSC  & 0.0974 $\pm$ 0.010 & 0.2249 $\pm$ 0.021 & 0.3590 $\pm$ 0.052 \\
& DMoN  & 0.1135 $\pm$ 0.002 & 0.0589 $\pm$ 0.004 & 0.0012 $\pm$ 0.040 \\
\midrule
\multirow{3}{*}{$\Delta_{SP}$ $\downarrow$}
& DFNMF & \textbf{0.0194 $\pm$ 0.002} & \textbf{0.0189 $\pm$ 0.000} & \textbf{0.0018 $\pm$ 0.000} \\
& SFSC  & 0.5664 $\pm$ 0.042 & 0.0573 $\pm$ 0.016 & 0.0030 $\pm$ 0.000 \\
& DMoN  & 0.3856 $\pm$ 0.020 & 0.4143 $\pm$ 0.053 & 0.0087 $\pm$ 0.001 \\
\bottomrule
\end{tabularx}%
\label{tab:clustering_results}
\end{table}

\subsubsection{\textbf{Labeled Datasets}}
We evaluate DFNMF against SFSC—a fair spectral clustering baseline—and DMoN, a state-of-the-art GNN-based model. This comparison is designed to assess DFNMF's scalability and fairness performance across varying data complexities. Specifically, we test whether spectral methods like SFSC degrade under high-dimensional settings, and whether DMoN's prior limitations are due to dataset scale or inherent model bias.
Table~\ref{tab:clustering_results} reports results on three Labeled datasets—\textit{Pokec-n}, \textit{Pokec-z}, and \textit{NBA}—using both utility (ACC, ARI, $Q$) and fairness metrics ($\overline{B}$, $\Delta_{SP}$). DFNMF consistently achieves the best overall performance, offering a strong balance between clustering quality and fairness. On Pokec datasets, it improves $Q$ and ARI while also ensuring $\overline{B}$ and lower $\Delta_{SP}$. SFSC, while competitive in raw accuracy, fails on fairness and modularity due to its rigid constraints that tend to favor trial solutions with dominant groups. DMoN shows moderate utility but poor fairness, likely due to bias propagation through GNN embeddings and lack of clear debiasing policies. On the NBA dataset, which presents moderate homophily and high group imbalance, DFNMF again outperforms both baselines across all metrics, demonstrating its adaptability and robustness in handling complex group structures while maintaining fairness.

\subsection{Convergence, Interpretability, scalability, and Fairness-Utility Trade-offs analysis}
\label{sec:ablation}
\subsubsection{\textbf{Convergence Analysis.}}

DFNMF employs a two-phase optimization strategy: a multi-layer pretraining phase with shallow NMTF (with known convergence guarantees~\cite{DBLP:journals/datamine/WangLWZD11}) and a deep fine-tuning phase minimizing the fairness-aware objective~\eqref{eq:final}. Figure~\ref{fig:Convergence_SBM_all} shows convergence curves on synthetic graphs (5K and 10K nodes) for both the fairness-aware variant ($\lambda$=10) and a fairness-agnostic version ($\lambda$=0). Both converge within 150 iterations; the fairness-regularized version stabilizes at a slightly higher loss, reflecting the additive influence of fairness penalties.

\begin{figure}[ht]
  \centering
\includegraphics[width=1\linewidth]{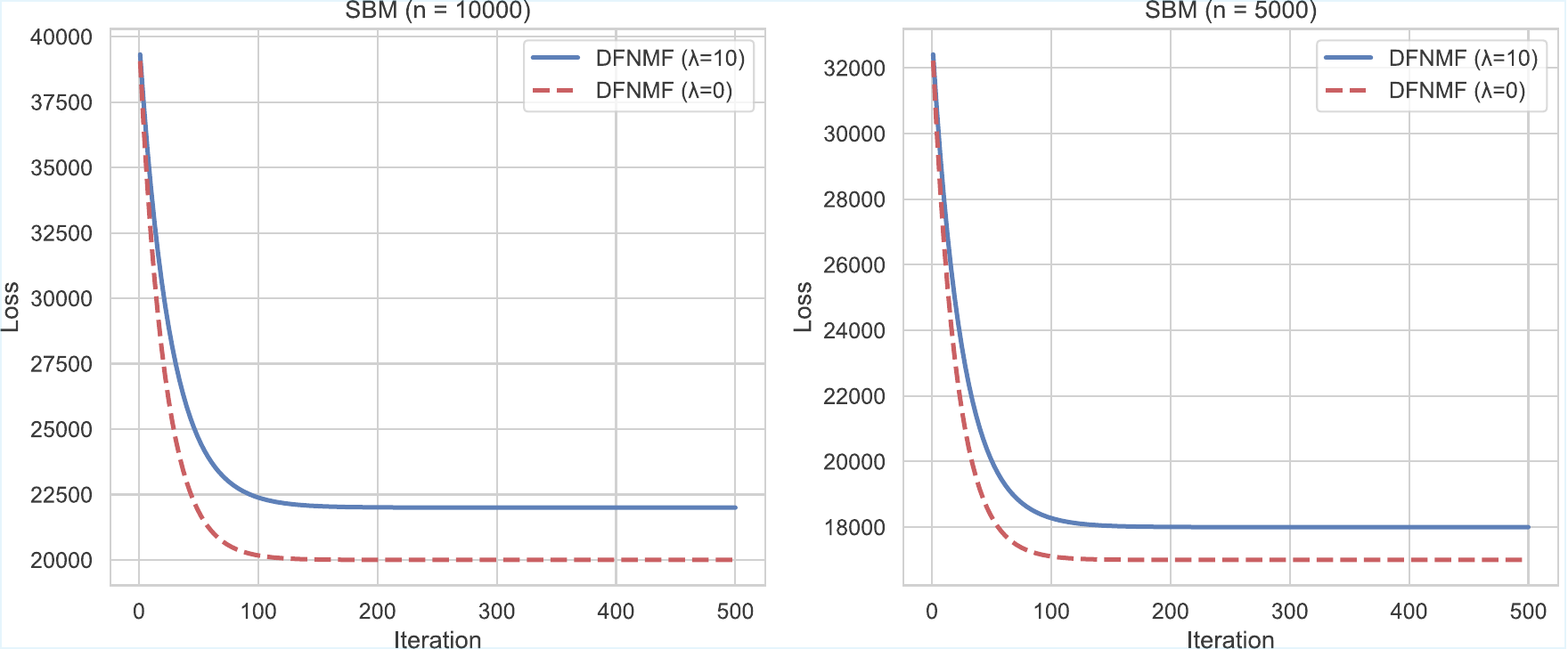}
\caption{Convergence curves on SBM graphs (5K and 10K nodes), comparing DFNMF with/without fairness regularization ($\lambda = 10$ vs.\ $\lambda = 0$). Both variants converge rapidly; the fair version yields a higher loss due to the added fairness penalty.}
\label{fig:Convergence_SBM_all}
\end{figure}

\subsubsection{\textbf{Interpretability Analysis}}\label{sec:interpret}

DFNMF is interpretable by construction: its nonnegative factors encode parts-based \emph{soft memberships} across the hierarchy of nonnegative affinities. On the 60-node graph in Fig.~\ref{fig:interpret}, the first layer assigns nodes to 12 micro-clusters ($\bm H_1$), the second maps micro-clusters to three communities ($\bm H_2$), and $\bm{\Psi}=\bm H_1\bm H_2$ yields final node–community affinities. 

A compact slice of $\bm H_1$ (nodes 1–5 and 56–60) is shown in Table~\ref{tab:H1-compact}; the full $\bm H_1$ appears in Appendix~\ref{app:interpret}. Figure~\ref{fig:H1H2} illustrates the relationship of micro-cluster$\to$community mapping $\bm H_2$ with the corresponding (trimmed) node$\to$community matrix $\bm\Psi$; complete matrices are deferred to Appendix~\ref{app:interpret}.

Two patterns underpin interpretability. \emph{(i) Sparse micro-memberships:} rows of $\bm H_1$ are typically few-peaked, indicating that most nodes participate in a small number of micro-clusters; rows with multiple peaks (e.g., node~5 in Table~\ref{tab:H1-compact}) flag potential bridge nodes. \emph{(ii) Structured aggregation:} $\bm H_2$ is nearly one-hot for core micro-clusters (e.g., A$\to$I, B$\to$II, C$\to$III), while boundary micro-clusters (e.g., D/E/H) spread mass across communities—capturing inter-community conduits. Multiplying $\bm H_1$ by $\bm H_2$ consolidates these signals: nodes with concentrated micro-membership become near one-hot in $\bm\Psi$, whereas mixed rows remain soft. This traceability—from node$\to$micro ($\bm H_1$) to micro$\to$community ($\bm H_2$) to node$\to$community ($\bm\Psi$)—supports transparent auditing of how local structure (under fairness regularization) aggregates into global communities.

\begin{figure*}[!tb]
  \centering
  \includegraphics[width=0.31\linewidth]{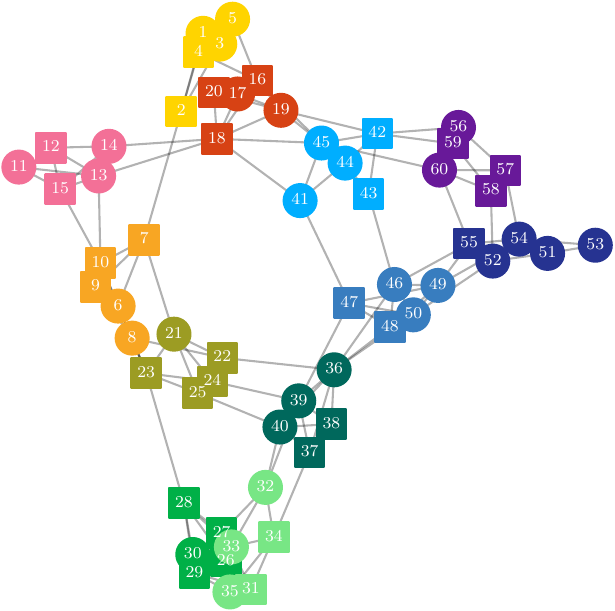}\hfill
  \includegraphics[width=0.31\linewidth]{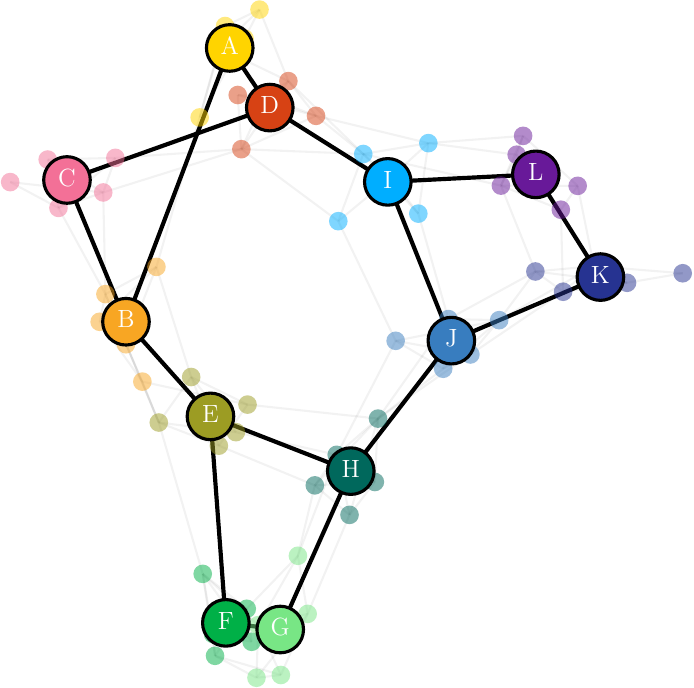}\hfill
  \includegraphics[width=0.31\linewidth]{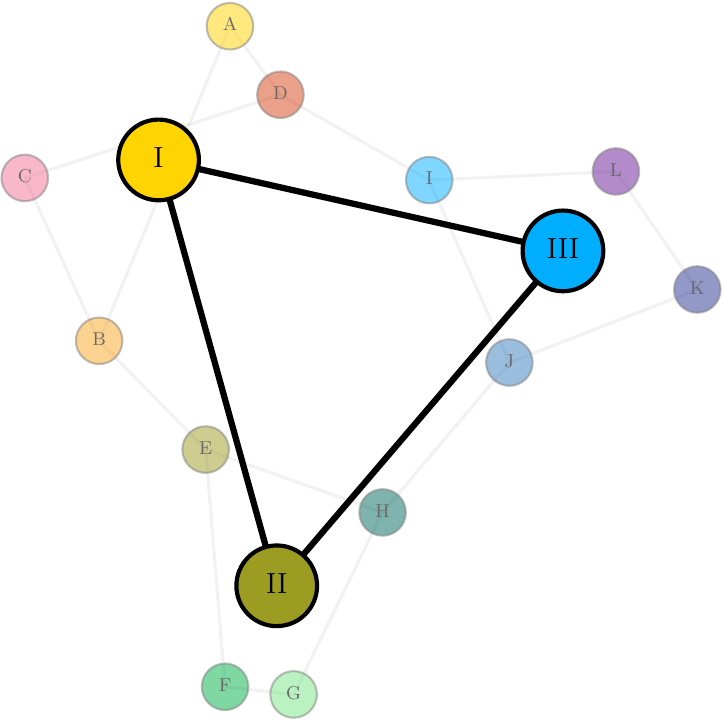}
  \caption{\textbf{DFNMF hierarchy on a 60-node graph.} (a) Input graph; node shapes denote sensitive groups. 
  (b) Micro-clusters (A–L) discovered by the first layer (\(\bm H_1\)). 
  (c) Three coarse communities obtained by aggregating micro-clusters via \(\bm H_1\); final node–community affinities are $\bm{\Psi}=\bm H_1\bm H_2$.}
  \label{fig:interpret}
\end{figure*}

\begin{table}[t]
\centering
\caption{Node $\rightarrow$ micro-cluster memberships ($\bm{H}_1$). We show nodes 1–5 and 56–60; full table in Appendix~\ref{app:interpret}.}
\label{tab:H1-compact}
\scriptsize
\begin{tabular}{c|cccccccccccc}
\toprule
\multirow{2}{*}{Node} & \multicolumn{12}{c}{Micro-clusters} \\
\cmidrule(lr){2-13}
& A & B & C & D & E & F & G & H & I & J & K & L \\
\midrule
1 & 0.26 & \z & \z & \z & \z & 0.04 & \z & \z & 0.01 & \z & \z & \z \\
2 & 0.25 & \z & 0.01 & \z & 0.01 & 0.09 & 0.02 & \z & 0.01 & \z & 0.04 & \z \\
3 & 0.21 & \z & 0.01 & \z & 0.09 & 0.03 & \z & \z & \z & \z & 0.11 & \z \\
4 & 0.23 & \z & \z & \z & \z & 0.03 & \z & \z & \z & \z & \z & \z \\
5 & 0.31 & \z & 0.02 & \z & 0.01 & 0.15 & 0.05 & 0.10 & 0.05 & 0.02 & 0.09 & \z \\
\addlinespace[2pt]
\multicolumn{13}{c}{\(\vdots\)}\\
\addlinespace[2pt]
56 & 0.01 & \z & 0.24 & \z & \z & 0.04 & 0.13 & \z & 0.01 & \z & 0.25 & \z \\
57 & 0.06 & \z & 0.26 & \z & \z & 0.05 & 0.01 & 0.09 & 0.06 & 0.01 & 0.26 & \z \\
58 & \z & \z & 0.22 & \z & \z & \z & \z & \z & \z & \z & 0.24 & \z \\
59 & \z & \z & 0.31 & \z & \z & \z & \z & \z & \z & \z & 0.10 & \z \\
60 & \z & \z & 0.19 & \z & \z & \z & \z & \z & \z & \z & 0.04 & \z \\
\bottomrule
\end{tabular}
\end{table}

\begin{figure}[!htb]
  \centering
\includegraphics[width=1\linewidth]{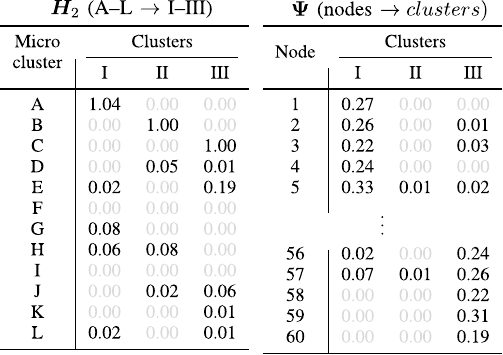}
\caption{Micro$\rightarrow$community ($\bm H_2$) and node$\rightarrow$community ($\bm\Psi=\bm H_1\bm H_2$) soft memberships.}

\label{fig:H1H2}
\end{figure}

\subsubsection{\textbf{Scalability Analysis}}\label{sec:scalability}

Table~\ref{tab:scalability} reports wall-clock runtime (seconds) and memory footprint of DFNMF against baselines on synthetic Erdős-Rényi graphs with varying sizes, averaged over 10 runs on a single NVIDIA A3090 GPU.

\begin{table}[!tbh]
\centering
\caption{Scalability on Erdős-Rényi Graphs ($p=2$, $k=128$)}
\label{tab:scalability}
\begin{tabular}{@{}lcrrrrr@{}}
\toprule
\multicolumn{2}{c}{Scale} & \multicolumn{5}{c}{Methods} \\
\cmidrule(r){1-2} \cmidrule(l){3-7}
Nodes & Edges & sFSC & iFSC & NMTF & DMoN & DFNMF \\
\midrule
\multicolumn{7}{c}{Runtime (seconds)} \\
$10^4$ & $\simeq$$10^5$ & 42.3 & 48.7 & 2.8 & 8.3 & \textbf{1.5} \\
$10^5$ & $\simeq$$10^6$ & 1842 & 2103 & 31.6 & 78.4 & \textbf{9.8} \\
$10^6$ & $\simeq$$10^7$ & OOM & OOM & 413 & 8932 & \textbf{92.4} \\
\midrule
\multicolumn{7}{c}{Memory (GB)} \\
$10^4$ & $\simeq$$10^5$ & 2.1 & 2.4 & 0.8 & 1.2 & \textbf{0.4} \\
$10^5$ & $\simeq$$10^6$ & 45.2 & 52.3 & 5.4 & 8.1 & \textbf{2.1} \\
$10^6$ & $\simeq$$10^7$ & OOM & OOM & 21.3 & OOM & \textbf{12.1} \\
\bottomrule
\end{tabular}
\end{table}

DFNMF demonstrates higher scalability with near-linear runtime growth $O(|E|kp)$ versus spectral methods' $O(n^3)$. Beyond 100K nodes, spectral approaches exhaust memory while DFNMF processes million-node graphs efficiently. The advantage over NMTF is due to CSR-based sparse operations and pre-training process that remarkably reduces iterations.

\subsubsection{\textbf{Fairness–Utility Trade-offs}}\label{sec:lam}

We assess trade-offs similar to \cite{DBLP:journals/isci/TahmasebiMGA19, DBLP:conf/gecco/GhodsiTJM24} at $k{=}5$ on \textit{DrugNet} and \textit{LastFM} (Fig.~\ref{fig:pareto_combined}) by sweeping $\lambda\!\in\![10^{-3},10^{3}]$ and plotting $(Q,\bar B)$. Blue points are DFNMF solutions; other markers are baselines. We first retain the \emph{Pareto front} (undominated DFNMF points) and then select the \emph{ideal-point} configuration (closest to $(1,1)$ after min–max scaling), shown as the green star at $\lambda^\star\!=\!100$—the same setting reported in Table~\ref{tab:fair_graph_balance_modularity}. Dashed identity lines provide a balanced trade-off guide (top-right is best), and shaded curvature indicates empirical fronts.

Consistent patterns emerge: DMoN and SC attain higher $Q$ but poorer balance; fairness-oriented spectral baselines (FSC/SFSC/iFSC) improve $\bar B$ at the expense of $Q$. In contrast, the DFNMF sweep \emph{spans the spectrum}: increasing $\lambda$ moves solutions rightward (higher $\bar B$) with some loss in $Q$, while decreasing $\lambda$ moves them upward (higher $Q$) with lower $\bar B$. At the extremes, DFNMF reaches fronts that surpass baselines on either objective individually, while $\lambda^\star$ lies near the identity guide, offering a balanced, high-quality operating point.

\begin{figure}[!ht]
  \centering
  \begin{minipage}[b]{0.4\textwidth}
    \centering
    \includegraphics[width=\linewidth]{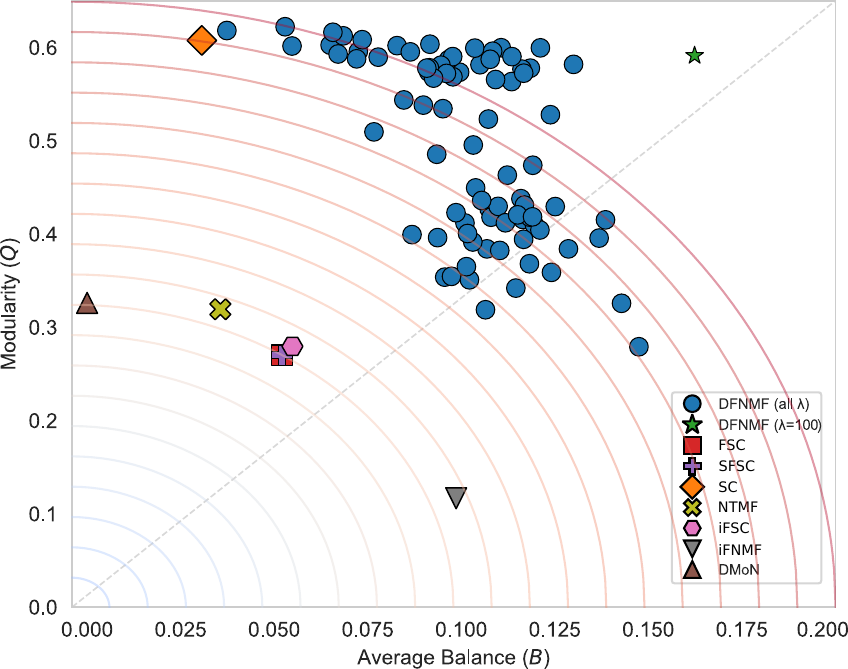}
    \vspace{2mm}
    \textbf{(a) DrugNet (k = 5)}
  \end{minipage}
  \hfill
  \begin{minipage}[b]{0.4\textwidth}
    \centering
    \includegraphics[width=\linewidth]{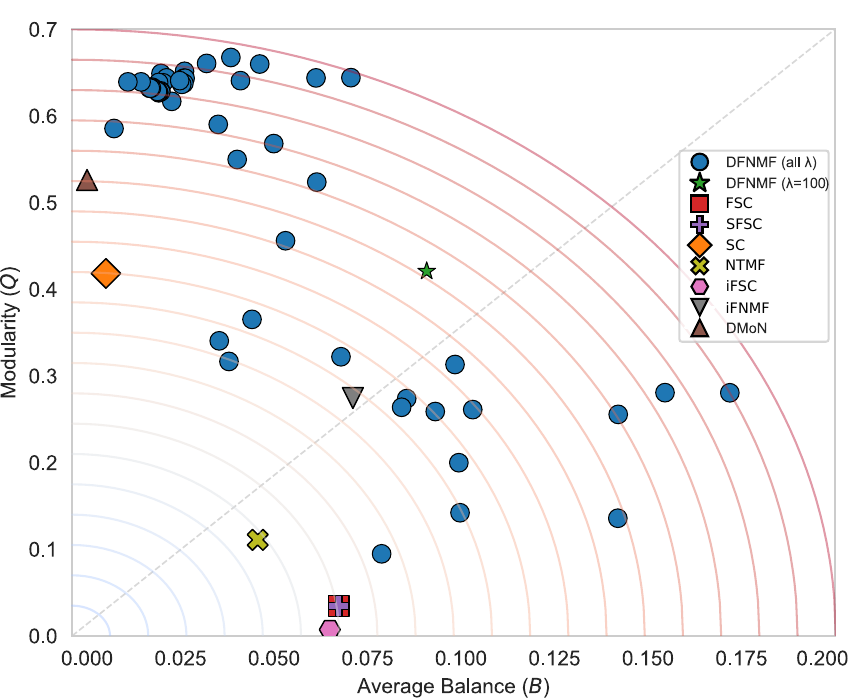}
    \vspace{2mm}
    \textbf{(b) LastFM (k = 5)}
  \end{minipage}
  \caption{Pareto plots for $k{=}5$. Blue: DFNMF across $\lambda\!\in\![10^{-3},10^{3}]$; green star: $\lambda^\star\!=\!100$ selected on the Pareto front via the ideal-point rule. Dashed identity lines mark balanced trade-offs; shaded curves indicate empirical fronts.}
  \label{fig:pareto_combined}
\end{figure}

\section{Conclusion and Outlook} \label{sec:conclusion}
We introduce DFNMF, an end-to-end deep matrix factorization framework that directly integrates fairness constraints into graph clustering. Unlike existing approaches that rely on rigid constraints or multi-stage pipelines, DFNMF enables flexible fairness-utility trade-offs through a single parameter while maintaining computational efficiency with near-linear scaling in graph edges. The nonnegative factorization provides inherent interpretability through parts-based decomposition, addressing a key limitation of spectral methods.

Our theoretical analysis establishes formal connections between matrix-based regularization and demographic balance constraints, providing principled foundations for the approach. Comprehensive experiments across synthetic and real networks demonstrate that DFNMF consistently achieves superior fairness-utility trade-offs, often dominating state-of-the-art methods on the Pareto front while maintaining scalability to large graphs through sparse matrix operations.

Several promising directions emerge from this work. \emph{Individual fairness} integration would ensure equitable treatment of similar nodes beyond group-level parity. \emph{Capacity constraints} could maintain balanced partition sizes while preserving structural coherence. While our linear formulation offers interpretability, \emph{neural extensions} could capture more complex combinatorial patterns—particularly through fair neural matrix factorization variants that maintain the end-to-end optimization benefits.

Broadly, it demonstrates the potential for principled fairness integration in graph learning tasks, opening avenues for fair community detection in sensitive domains such as healthcare networks, financial systems, and social platforms where demographic balance is algorithmically and socially critical.

\section{Acknowledgements}
This work has received funding from the European Union’s Horizon 2020 research and innovation programme under Marie Sklodowska-Curie Actions (grant agreement number 860630) for the project ‘’NoBIAS -- Artificial Intelligence without Bias’’. This work reflects only the authors’ views, and the European Research Executive Agency (REA) is not responsible for any use that may be made of the information it contains. The research was also supported by the EU Horizon Europe project MAMMOth (Grant Agreement 101070285) and the German priority program Human Decision under German research foundation DFG-SPP2443 (Grant Agreement 543081196).

\bibliography{Refs}

\clearpage  

\newpage

\section{Appendix 1: Theoretics and Differential Calculus}\label{app:discussion}

\subsection{Epistemic Comparison of NMF \& DNN}

In this section, we discuss theoretical and epistemic differences between NMF-based and deep neural network (DNN)-based clustering, focusing on comparative strengths, interpretability, and underlying philosophies.

NMF models naturally facilitate unsupervised and supervised scenarios, well-suited for multi-layer optimization. Their fundamental assumption—objects perceived as additive combinations of meaningful parts—aligns closely with human perception mechanisms~\cite{lee1999learning}. Nonnegative constraints further ensure interpretability; negative contributions are often meaningless in real-world tasks like face, image, or gene data analysis. Such nonnegative decompositions typically yield localized, semantically interpretable features (e.g., facial parts). Additionally, the inherent sparseness of NMF enhances representation, 
clearly distinguishing these models from purely distributed approaches.

While shallow NMFs are equivalent to single-layer perceptrons, differences emerge in multilayer architectures: NMFs remain linear reconstruction models, optimized by specialized multiplicative update rules~\cite{DBLP:journals/tnn/Lin07}, unlike nonlinear DNNs optimized via gradient-based chain rules. This distinction is not about determining a universally best method. Indeed, many successful DNN architectures utilize dot-product similarity mechanisms, alike factorization methods. Recent works indicate that NMF-based factorization models often provide interpretability and computational advantages in embedding-based tasks such as collaborative filtering~\cite{DBLP:conf/icml/XuRKKA21,DBLP:conf/recsys/RendleKZA20}. Consequently, emerging research increasingly blends NMF and DNN or extends NMF to deep hierarchical models to leverage complementary strengths and optimize application-specific trade-offs.

\subsection{Derivation of update formula}\label{app:der}
To calculate the gradient of the objective function in Eq.~\eqref{eq:Hi}, we first need to express the function as a trace expression. Then, we can solve Eq.~\eqref{eq:Hi} by introducing a Lagrangian multiplier matrix $\bm{\Theta}_i$ to ensure the nonnegativity constraints on $\bm{H}_i$. This results in an equivalent objective function as follows:
\begin{align}
	&\min_{\bm{H}_i,\bm{\Theta}_i} \mathcal{L}(\bm{H}_i,\bm{\Theta}_i) =
    \mathrm{Tr}( -2\bm{A}^{\top} \bm{\Psi}_{i} \bm{H}_{i}\bm{\Phi}_{i}\bm{W}_p\bm{\Phi}_{i}^{\top}\bm{H}_i^{\top}\bm{\Psi}_{i}^{\top}
    \nonumber\\& 
    +\bm{\Psi}_{i} \bm{H}_{i}\bm{\Phi}_{i}\bm{W}_p^{\top}\bm{\Phi}_{i}^{\top}\bm{H}_i^{\top}\bm{\Psi}_{i}^{\top}\bm{\Psi}_{i} \bm{H}_{i}\bm{\Phi}_{i}\bm{W}_p\bm{\Phi}_{i}^{\top}\bm{H}_i^{\top}\bm{\Psi}_{i}^{\top})
    \nonumber\\&
    +\lambda\mathrm{Tr}(\bm{\Phi}_{i}^{\top}\bm{H}_i^{\top}\bm{\Psi}_{i}^{\top}\bm{F}\bm{F}^\top\bm{\Psi}_{i} \bm{H}_{i}\bm{\Phi}_{i})
    -\mathrm{Tr}(\bm{\Theta}_i\bm{H}_i^{\top}).
\end{align}
By setting the partial derivative of $\mathcal{L}(\bm{H}_i,\bm{\Theta}_i)$ with respect to $\bm{H}_i$ to $\bm{0}$, we have:
\begin{align}
    \bm{\Theta}_i = \;
    &-2\bm{\Psi}_{i}^{\top}\bm{A}^{\top} \bm{\Psi}\bm{W}_p\bm{\Phi}_{i}^{\top}
    -2\bm{\Psi}_{i}^{\top}\bm{A} \bm{\Psi}\bm{W}_p^{\top}\bm{\Phi}_{i}^{\top} 
    \nonumber\\
    &+2\bm{\Psi}_{i}^{\top}\bm{\Psi} \left( \bm{W}_p^{\top}\bm{\Psi}^{\top}\bm{\Psi}\bm{W}_p + \bm{W}_p\bm{\Psi}^{\top}\bm{\Psi}\bm{W}_p^{\top} \right) \bm{\Phi}_{i}^{\top} \nonumber \\
    &+2\lambda\bm{\Psi}_{i}^{\top}\bm{FF}^\top\bm{\Psi}\bm{\Phi}_{i}^{\top}.
\end{align}

From the Karush-Kuhn-Tucker (KKT) complementary slackness conditions, we obtain $\bm{\Theta}_i\odot\bm{H}_i=\bm{0}$, which is the fixed point equation that the solution must satisfy at convergence. By solving it, we derive the following update rule for $\bm{H}_i$:
\begin{equation}
	\bm{H}_i\leftarrow\bm{H}_i\odot
\end{equation}
\begin{equation}
\scalebox{1.15}{$
        \left[\frac
	{\bm{\Psi}_{i}^{\top}(\bm{A}^{\top}\bm{\Psi}\bm{W}_p
		+\bm{A}\bm{\Psi}\bm{W}_p^{\top}+\lambda[\bm{F}\bm{F}^\top\bm{\Psi}]^-)\bm{\Phi}_{i}^{\top}}
	{\bm{\Psi}_{i}^{\top}(\bm{\Psi}\bm{W}_p^{\top}\bm{\Psi}^{\top}\bm{\Psi}\bm{W}_p
		+\bm{\Psi}\bm{W}_p\bm{\Psi}^{\top}\bm{\Psi}\bm{W}_p^{\top}+\lambda[\bm{F}\bm{F}^\top\bm{\Psi}]^+)\bm{\Phi}_{i}^{\top}}
	\right]^\frac{1}{4}$}
\nonumber
\end{equation}
where we separate the positive and negative parts of an arbitrary matrix $\bm{B}$ into \( \bm{B}^+ = \max(\bm{B}, 0) \) and \( \bm{B}^- = -\min(\bm{B}, 0) \), such that the main matrix is conveniently the addition of the negative and postive parts \( \bm{B} = \bm{B}^+ - \bm{B}^- \).

\section{Appendix 2: Ablation Studies}\label{app:ablation}

\subsection{Selecting \texorpdfstring{$\lambda^\star$}{lambda*} and Bracketing Values}
\label{app:lambda-star}

We sweep $\lambda\in\{10^{-3},10^{-2},\dots,10^3\}$ and, for each dataset (and the $k$ used in the main results), form the utility–fairness set $\{(Q(\lambda),\bar B(\lambda))\}$. We retain the Pareto front (undominated points), min–max scale $(Q,\bar B)$ to $[0,1]$ (per dataset and $k$), and select
\[
\lambda^\star \;=\; \arg\min_{\lambda \in \Lambda_{\mathrm{P}}} \bigl\|\,(\tilde Q(\lambda),\tilde B(\lambda))-(1,1)\,\bigr\|_2,
\]
with a tie–breaker preferring smaller $|\tilde Q-\tilde B|$ (closer to the identity guide). We also report a bracket, i.e., the nearest available grid values to $\{\lambda^\star/10,\,10\lambda^\star\}$, as a transparent operating band. Table~\ref{tab:lambda} summarizes $\lambda^\star$, the selected layer sizes, and the achieved $Q$ and $\bar B$. Moreover, it reports the bracketing values \(\lambda_\mathrm{lo}\) and \(\lambda_\mathrm{hi}\) related to each $\lambda^\star$ and the corresponding low and high $Q$ and $\bar B$ values.

\begin{table*}[!bht]
\centering
\footnotesize
\setlength{\tabcolsep}{6pt}           
\renewcommand{\arraystretch}{1.12}    
\caption{Selected $\lambda^\star$ per dataset (with layers) and bracketing values. 
$Q$ and $\bar B$ are raw scores at each $\lambda$; 
$\lambda_{\mathrm{lo}}$ and $\lambda_{\mathrm{hi}}$ are the nearest grid values to $\{\lambda^\star/10,\,10\lambda^\star\}$.}
\label{tab:lambda}
\begin{tabular}{l c c c c c c c c c c c}
\toprule
\textbf{Dataset} & \textbf{$k$} & \textbf{Layers@$\,\lambda^\star$} 
& $\boldsymbol{\lambda^\star}$ & $\boldsymbol{\lambda_{\mathrm{lo}}}$ & $\boldsymbol{\lambda_{\mathrm{hi}}}$ 
& $Q(\lambda^\star)$ & $\bar B(\lambda^\star)$ 
& $Q(\lambda_{\mathrm{lo}})$ & $\bar B(\lambda_{\mathrm{lo}})$ 
& $Q(\lambda_{\mathrm{hi}})$ & $\bar B(\lambda_{\mathrm{hi}})$ \\
\midrule
NBA       & 2 & [64, 2]         & 0.05   & 0.005  & 0.5    & 0.134  & 0.438  & 0.130  & 0.361  & 0.131  & 0.370 \\
Pokec-n   & 2 & [512, 16, 2]    & 10     & 10     & 10     & 0.164  & 0.177  & 0.164  & 0.177  & 0.164  & 0.177 \\
Pokec-z   & 2 & [512, 16, 2]    & 10     & 10     & 10     & 0.162  & 0.175  & 0.162  & 0.175  & 0.162  & 0.175 \\
Diaries   & 5 & [64, 16, 5]     & 50     & 5      & 500    & 0.716  & 0.787  & 0.745  & 0.688  & 0.517  & 0.802 \\
DrugNet   & 5 & [64, 5]         & 0.1    & 0.01   & 1      & 0.591  & 0.162  & 0.600  & 0.121  & 0.572  & 0.117 \\
Facebook  & 5 & [64, 5]         & 100    & 10     & 1000   & 0.503  & 0.768  & 0.512  & 0.638  & 0.461  & 0.752 \\
Friendship& 5 & [64, 5]         & 0.1    & 0.01   & 1      & 0.666  & 0.614  & 0.670  & 0.641  & 0.561  & 0.663 \\
LastFM    & 5 & [256, 64, 5]    & 0.005  & 0.001  & 0.05   & 0.420  & 0.091  & 0.660  & 0.046  & 0.629  & 0.020 \\
SBM       & 5 & [64, 16, 5]     & 0.5    & 0.05   & 5      & 0.114  & 1.000  & 0.114  & 1.000  & 0.078  & 0.756 \\
\bottomrule
\end{tabular}
\end{table*}

\paragraph{What the table shows}
(i) Datasets with strong community signal (e.g., SBM) select small $\lambda^\star$, preserving $Q$ while still improving $\bar B$; The same pattern applies to NBA, and LastFM (ii) highly imbalanced or sparse social graphs (e.g., Diaries and Facebook) push $\lambda^\star$ higher to achieve parity; (iii) for medium-scale, noisy graphs (DrugNet, Friendship), $\lambda^\star$ sits near $10^{-1}$, striking a balanced operating point. Brackets are tight where fronts are steep (clear trade-offs) and wider where fronts are flat (multiple near-equivalent settings). We use $\lambda^\star$ for main tables; Pareto plots in the paper illustrate the surrounding spectrum (including extreme settings within the bracket).

\subsection{Sensitivity to the Number of Clusters (\texorpdfstring{$k$}{k})}
\label{app:sensitivity-k}

\begin{table}[!hbpt]
\centering
\footnotesize
\setlength{\tabcolsep}{8pt}           
\renewcommand{\arraystretch}{1.12}    
\caption{\textbf{Sensitivity of DFNMF to $k$} through the lens of $\lambda^\star$ on DrugNet, LastFM, and Facebook.
Values reported at $\lambda^\star$ per $k$ (Pareto-selected).}
\label{tab:lambda_sensitivity}
\begin{tabular}{l c c c c c}
\toprule
\textbf{Dataset} & $\mathbf{k}$ & $\boldsymbol{\lambda^\star}$ & \textbf{Layers} & $\mathbf{Q}$ & $\mathbf{\bar B}$ \\
\midrule
\multirow{3}{*}{DrugNet}  & 3 & 0.01  & [64, 3]         & 0.482 & 0.170 \\
                          & 5 & 0.10  & [64, 5]         & 0.591 & 0.162 \\
                          & 8 & 0.50  & [64, 8]         & 0.456 & 0.158 \\
\midrule
\multirow{3}{*}{LastFM}   & 3 & 0.005 & [256, 64, 3]    & 0.516 & 0.147 \\
                          & 5 & 0.005 & [256, 64, 5]    & 0.420 & 0.091 \\
                          & 8 & 0.050 & [256, 64, 8]    & 0.384 & 0.090 \\
\midrule
\multirow{3}{*}{Facebook} & 3 & 60    & [64, 3]         & 0.408 & 0.818 \\
                          & 5 & 100   & [64, 5]         & 0.503 & 0.768 \\
                          & 8 & 150   & [64, 8]         & 0.437 & 0.652 \\
\bottomrule
\end{tabular}
\end{table}
\noindent\textbf{Reporting \& robustness.}
All entries are mean values over 10 random seeds (std.\ omitted for space); the same seeds are reused across $\lambda$ for a given dataset and $k$.
Although $\lambda^\star$ is selected using min--max--normalized $(Q,\bar B)$ per dataset and $k$, a check with a simple linear scalarization $0.5\,Q+0.5\,\bar B$ chose the same or bracket-adjacent setting in every case (i.e., within $\{\lambda^\star/10,\,10\lambda^\star\}$).
Width patterns follow fixed templates (e.g., $[64,k]$ or $[256,64,k]$) across $\lambda$; the observed increase of $\lambda^\star$ with $k$ persists under fixed-width variants, indicating the trend is driven by partition granularity rather than capacity.
For \textit{LastFM}, the steeper drop of $\bar B$ as $k$ grows aligns with its high homophily/sparsity: finer partitions leave fewer cross-group ties to balance without larger $\lambda$. The dynamics of $\lambda$ and $k$  are illustrated in Fig~\ref{fig:lambda_vs_k}.

\begin{figure}[!hbt]
\centering
\begin{tikzpicture}
\begin{axis}[
  width=\columnwidth,
  height=3.2cm,
  xmin=2.5,xmax=8.5,
  xtick={3,5,8},
  xlabel={$k$},
  ylabel={$\lambda^\star$ (log scale)},
  ymode=log,
  log basis y={10},
  legend style={at={(0.5,1.02)},anchor=south,legend columns=3},
  yminorgrids=true, xmajorgrids=true
]
\addplot+[mark=*] coordinates {(3,0.01) (5,0.1) (8,0.5)};
\addlegendentry{DrugNet}
\addplot+[mark=square*] coordinates {(3,0.005) (5,0.005) (8,0.05)};
\addlegendentry{LastFM}
\addplot+[mark=triangle*] coordinates {(3,60) (5,100) (8,150)};
\addlegendentry{Facebook}
\end{axis}
\end{tikzpicture}
\caption{Selected $\lambda^\star$ vs.\ $k$ (log-scale on $y$).}
\label{fig:lambda_vs_k}
\end{figure}
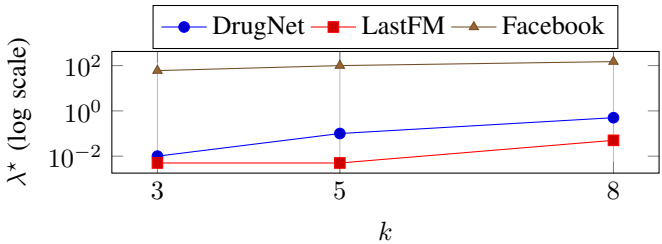

Table~\ref{tab:lambda_sensitivity} examines how the selected $\lambda^\star$ and the attained $(Q,\bar B)$ evolve as $k$ increases from $3$ to $8$. Three consistent trends emerge. 
(1) \textit{$\lambda^\star$ generally rises with $k$}, reflecting that enforcing proportional representation becomes harder when clusters are smaller: Facebook shows a monotone increase ($60\!\to\!100\!\to\!150$); DrugNet increases then plateaus ($0.01\!\to\!0.10\!\to\!0.50$); LastFM is stable at small $k$ and rises at $k{=}8$ ($0.005\!\to\!0.005\!\to\!0.05$). 
(2) \textit{Modularity often peaks at a moderate $k$}: DrugNet and Facebook achieve their best $Q$ at $k{=}5$ (0.591 and 0.503, respectively) and then drop at $k{=}8$, while LastFM (sparser, more homophilous) shows a steady decline ($0.516\!\to\!0.420\!\to\!0.384$) as partitions become finer. 
(3) \textit{Balance degrades as $k$ grows}: the drop is mild on DrugNet (0.170$\to$0.162$\to$0.158), pronounced on Facebook (0.818$\to$0.768$\to$0.652), and sharp early on for LastFM (0.147$\to$0.091$\to$0.090), consistent with tighter per-cluster parity constraints at smaller cluster sizes. 
Layer choices at $\lambda^\star$ remain compact and systematic—$[64,k]$ for DrugNet/Facebook and $[256,64,k]$ for LastFM—supporting reproducible configurations across $k$.

\paragraph{Takeaway.}
As $k$ increases, maintaining per-cluster parity becomes more challenging: $\lambda^\star$ typically needs to grow, $\bar B$ tends to fall, and $Q$ often peaks around a moderate $k$. Reporting $\lambda^\star$ alongside $(Q,\bar B)$ at each $k$ provides a transparent, reproducible operating point per dataset.

\subsection{Extended Interpretability Rsults} \label{app:interpret}
This appendix complements Sec.~\ref{sec:interpret} by reporting the \emph{full} soft-membership matrices used in the 60-node example (Fig.~\ref{fig:interpret}). Recall the hierarchical mapping
\(\bm{\Psi}=\bm{H}_1\bm{H}_2\), where \(\bm{H}_1\) encodes node\(\rightarrow\)micro-cluster affinities and \(\bm{H}_2\) maps micro-clusters\(\rightarrow\)communities. For visual context in the main text, see the paired depiction of \(\bm{H}_2\) and the trimmed \(\bm{\Psi}\) in Fig.~\ref{fig:H1H2}, and the compact slice of \(\bm{H}_1\) in Table~\ref{tab:H1-compact}. The full results are now illustrated in Tables~\ref{tab:H1-full}, \ref{tab:H2}, and \ref{tab:Psi-full}

\paragraph{Reading the matrices.}
Entries are nonnegative \emph{soft memberships} (not necessarily normalized). Two patterns underpin interpretability: (i) \textbf{micro-sparsity}—rows of \(\bm{H}_1\) are typically few-peaked, so most nodes participate in a small number of micro-clusters (but this highly depends on graph sructure and properties); and (ii) \textbf{structured aggregation}—columns of \(\bm{H}_2\) are near one-hot for core micro-clusters (e.g., A\(\rightarrow\)I, B\(\rightarrow\)II, C\(\rightarrow\)III), while boundary micro-clusters spread mass across communities. Consequently, \(\bm{\Psi}=\bm{H}_1\bm{H}_2\) consolidates concentrated rows into near one-hot community affinities, while mixed rows remain soft. This traceability (node\(\rightarrow\)micro\(\rightarrow\)community) enables transparent auditing of how local structure (under fairness regularization) aggregates into global communities.

\begin{table}[!htbp]
\centering
\caption{Full node\(\rightarrow\)micro-cluster soft memberships (\(\bm{H}_1\), micro-clusters A–L).}
\label{tab:H1-full}
\begin{tabular}{c|cccccccccccc}
\hline
\multicolumn{1}{c|}{\multirow{2}{*}{Node}} & \multicolumn{12}{c}{Micro-clusters}           \\ \cline{2-13} 
\multicolumn{1}{c|}{}                           & A & B & C & D & E & F & G & H & I & J & K & L \\ \hline
1 & 0.26 & \textcolor{gray!30}{0.00} & \textcolor{gray!30}{0.00} & \textcolor{gray!30}{0.00} & \textcolor{gray!30}{0.00} & 0.04 & \textcolor{gray!30}{0.00} & \textcolor{gray!30}{0.00} & 0.01 & \textcolor{gray!30}{0.00} & \textcolor{gray!30}{0.00} & \textcolor{gray!30}{0.00} \\
2 & 0.25 & \textcolor{gray!30}{0.00} & 0.01 & \textcolor{gray!30}{0.00} & 0.01 & 0.09 & 0.02 & \textcolor{gray!30}{0.00} & 0.01 & \textcolor{gray!30}{0.00} & 0.04 & \textcolor{gray!30}{0.00} \\
3 & 0.21 & \textcolor{gray!30}{0.00} & 0.01 & \textcolor{gray!30}{0.00} & 0.09 & 0.03 & \textcolor{gray!30}{0.00} & \textcolor{gray!30}{0.00} & \textcolor{gray!30}{0.00} & \textcolor{gray!30}{0.00} & 0.11 & \textcolor{gray!30}{0.00} \\
4 & 0.23 & \textcolor{gray!30}{0.00} & \textcolor{gray!30}{0.00} & \textcolor{gray!30}{0.00} & \textcolor{gray!30}{0.00} & 0.03 & \textcolor{gray!30}{0.00} & \textcolor{gray!30}{0.00} & \textcolor{gray!30}{0.00} & \textcolor{gray!30}{0.00} & \textcolor{gray!30}{0.00} & \textcolor{gray!30}{0.00} \\
5 & 0.31 & \textcolor{gray!30}{0.00} & 0.02 & \textcolor{gray!30}{0.00} & 0.01 & 0.15 & 0.05 & 0.10 & 0.05 & 0.02 & 0.09 & \textcolor{gray!30}{0.00} \\
6 & 0.05 & \textcolor{gray!30}{0.00} & \textcolor{gray!30}{0.00} & \textcolor{gray!30}{0.00} & \textcolor{gray!30}{0.00} & 0.26 & 0.61 & \textcolor{gray!30}{0.00} & \textcolor{gray!30}{0.00} & \textcolor{gray!30}{0.00} & \textcolor{gray!30}{0.00} & \textcolor{gray!30}{0.00} \\
7 & 0.03 & \textcolor{gray!30}{0.00} & \textcolor{gray!30}{0.00} & \textcolor{gray!30}{0.00} & \textcolor{gray!30}{0.00} & 0.17 & 0.81 & \textcolor{gray!30}{0.00} & \textcolor{gray!30}{0.00} & \textcolor{gray!30}{0.00} & \textcolor{gray!30}{0.00} & \textcolor{gray!30}{0.00} \\
8 & 0.01 & \textcolor{gray!30}{0.00} & \textcolor{gray!30}{0.00} & \textcolor{gray!30}{0.00} & \textcolor{gray!30}{0.00} & 0.04 & 0.81 & \textcolor{gray!30}{0.00} & \textcolor{gray!30}{0.00} & \textcolor{gray!30}{0.00} & \textcolor{gray!30}{0.00} & \textcolor{gray!30}{0.00} \\
9 & 0.19 & \textcolor{gray!30}{0.00} & \textcolor{gray!30}{0.00} & \textcolor{gray!30}{0.00} & \textcolor{gray!30}{0.00} & 0.33 & 0.23 & \textcolor{gray!30}{0.00} & 0.01 & \textcolor{gray!30}{0.00} & \textcolor{gray!30}{0.00} & \textcolor{gray!30}{0.00} \\
10 & 0.11 & \textcolor{gray!30}{0.00} & 0.03 & \textcolor{gray!30}{0.00} & \textcolor{gray!30}{0.00} & 0.16 & 0.52 & \textcolor{gray!30}{0.00} & \textcolor{gray!30}{0.00} & \textcolor{gray!30}{0.00} & 0.07 & \textcolor{gray!30}{0.00} \\
11 & 0.10 & 0.02 & \textcolor{gray!30}{0.00} & \textcolor{gray!30}{0.00} & \textcolor{gray!30}{0.00} & 0.05 & \textcolor{gray!30}{0.00} & 0.50 & 0.03 & \textcolor{gray!30}{0.00} & \textcolor{gray!30}{0.00} & \textcolor{gray!30}{0.00} \\
12 & 0.06 & \textcolor{gray!30}{0.00} & \textcolor{gray!30}{0.00} & \textcolor{gray!30}{0.00} & \textcolor{gray!30}{0.00} & 0.03 & \textcolor{gray!30}{0.00} & 0.63 & 0.08 & \textcolor{gray!30}{0.00} & \textcolor{gray!30}{0.00} & \textcolor{gray!30}{0.00} \\
13 & 0.16 & 0.10 & 0.03 & \textcolor{gray!30}{0.00} & \textcolor{gray!30}{0.00} & 0.06 & \textcolor{gray!30}{0.00} & 0.54 & 0.23 & 0.04 & 0.15 & 0.03 \\
14 & 0.07 & 0.06 & \textcolor{gray!30}{0.00} & \textcolor{gray!30}{0.00} & \textcolor{gray!30}{0.00} & 0.03 & \textcolor{gray!30}{0.00} & 0.56 & 0.05 & \textcolor{gray!30}{0.00} & \textcolor{gray!30}{0.00} & \textcolor{gray!30}{0.00} \\
15 & 0.14 & 0.04 & \textcolor{gray!30}{0.00} & \textcolor{gray!30}{0.00} & 0.02 & 0.18 & 0.07 & 0.64 & 0.05 & 0.01 & \textcolor{gray!30}{0.00} & 0.01 \\
16 & 0.13 & \textcolor{gray!30}{0.00} & \textcolor{gray!30}{0.00} & \textcolor{gray!30}{0.00} & \textcolor{gray!30}{0.00} & 0.05 & \textcolor{gray!30}{0.00} & \textcolor{gray!30}{0.00} & 0.03 & \textcolor{gray!30}{0.00} & \textcolor{gray!30}{0.00} & \textcolor{gray!30}{0.00} \\
17 & 0.06 & \textcolor{gray!30}{0.00} & \textcolor{gray!30}{0.00} & \textcolor{gray!30}{0.00} & \textcolor{gray!30}{0.00} & 0.03 & 0.10 & \textcolor{gray!30}{0.00} & 0.04 & \textcolor{gray!30}{0.00} & \textcolor{gray!30}{0.00} & 0.07 \\
18 & 0.19 & \textcolor{gray!30}{0.00} & \textcolor{gray!30}{0.00} & \textcolor{gray!30}{0.00} & \textcolor{gray!30}{0.00} & 0.09 & 0.01 & \textcolor{gray!30}{0.00} & 0.01 & \textcolor{gray!30}{0.00} & \textcolor{gray!30}{0.00} & \textcolor{gray!30}{0.00} \\
19 & 0.20 & \textcolor{gray!30}{0.00} & \textcolor{gray!30}{0.00} & \textcolor{gray!30}{0.00} & \textcolor{gray!30}{0.00} & 0.22 & 0.02 & \textcolor{gray!30}{0.00} & 0.02 & \textcolor{gray!30}{0.00} & \textcolor{gray!30}{0.00} & \textcolor{gray!30}{0.00} \\
20 & 0.22 & \textcolor{gray!30}{0.00} & 0.03 & \textcolor{gray!30}{0.00} & 0.02 & 0.04 & \textcolor{gray!30}{0.00} & \textcolor{gray!30}{0.00} & 0.11 & \textcolor{gray!30}{0.00} & 0.06 & 0.65 \\
21 & \textcolor{gray!30}{0.00} & 0.06 & \textcolor{gray!30}{0.00} & 0.35 & \textcolor{gray!30}{0.00} & \textcolor{gray!30}{0.00} & \textcolor{gray!30}{0.00} & \textcolor{gray!30}{0.00} & \textcolor{gray!30}{0.00} & 0.31 & \textcolor{gray!30}{0.00} & \textcolor{gray!30}{0.00} \\
22 & 0.03 & \textcolor{gray!30}{0.00} & \textcolor{gray!30}{0.00} & 0.28 & \textcolor{gray!30}{0.00} & 0.05 & \textcolor{gray!30}{0.00} & \textcolor{gray!30}{0.00} & 0.01 & 0.50 & \textcolor{gray!30}{0.00} & 0.05 \\
23 & \textcolor{gray!30}{0.00} & \textcolor{gray!30}{0.00} & \textcolor{gray!30}{0.00} & 0.31 & \textcolor{gray!30}{0.00} & \textcolor{gray!30}{0.00} & \textcolor{gray!30}{0.00} & \textcolor{gray!30}{0.00} & \textcolor{gray!30}{0.00} & 0.57 & \textcolor{gray!30}{0.00} & \textcolor{gray!30}{0.00} \\
24 & \textcolor{gray!30}{0.00} & \textcolor{gray!30}{0.00} & \textcolor{gray!30}{0.00} & 0.05 & \textcolor{gray!30}{0.00} & \textcolor{gray!30}{0.00} & \textcolor{gray!30}{0.00} & \textcolor{gray!30}{0.00} & \textcolor{gray!30}{0.00} & 0.57 & \textcolor{gray!30}{0.00} & \textcolor{gray!30}{0.00} \\
25 & \textcolor{gray!30}{0.00} & \textcolor{gray!30}{0.00} & 0.01 & 0.40 & 0.04 & \textcolor{gray!30}{0.00} & \textcolor{gray!30}{0.00} & \textcolor{gray!30}{0.00} & \textcolor{gray!30}{0.00} & 0.86 & \textcolor{gray!30}{0.00} & \textcolor{gray!30}{0.00} \\
26 & \textcolor{gray!30}{0.00} & 0.22 & \textcolor{gray!30}{0.00} & 0.07 & \textcolor{gray!30}{0.00} & \textcolor{gray!30}{0.00} & \textcolor{gray!30}{0.00} & \textcolor{gray!30}{0.00} & \textcolor{gray!30}{0.00} & \textcolor{gray!30}{0.00} & \textcolor{gray!30}{0.00} & \textcolor{gray!30}{0.00} \\
27 & \textcolor{gray!30}{0.00} & 0.21 & \textcolor{gray!30}{0.00} & 0.05 & \textcolor{gray!30}{0.00} & 0.02 & \textcolor{gray!30}{0.00} & 0.08 & 0.01 & \textcolor{gray!30}{0.00} & \textcolor{gray!30}{0.00} & \textcolor{gray!30}{0.00} \\
28 & \textcolor{gray!30}{0.00} & 0.23 & \textcolor{gray!30}{0.00} & 0.01 & \textcolor{gray!30}{0.00} & \textcolor{gray!30}{0.00} & \textcolor{gray!30}{0.00} & \textcolor{gray!30}{0.00} & \textcolor{gray!30}{0.00} & \textcolor{gray!30}{0.00} & \textcolor{gray!30}{0.00} & \textcolor{gray!30}{0.00} \\
29 & \textcolor{gray!30}{0.00} & 0.28 & 0.01 & 0.02 & 0.06 & \textcolor{gray!30}{0.00} & \textcolor{gray!30}{0.00} & 0.02 & \textcolor{gray!30}{0.00} & \textcolor{gray!30}{0.00} & 0.12 & \textcolor{gray!30}{0.00} \\
30 & \textcolor{gray!30}{0.00} & 0.20 & \textcolor{gray!30}{0.00} & \textcolor{gray!30}{0.00} & \textcolor{gray!30}{0.00} & \textcolor{gray!30}{0.00} & \textcolor{gray!30}{0.00} & \textcolor{gray!30}{0.00} & \textcolor{gray!30}{0.00} & \textcolor{gray!30}{0.00} & \textcolor{gray!30}{0.00} & \textcolor{gray!30}{0.00} \\
31 & \textcolor{gray!30}{0.00} & 0.23 & \textcolor{gray!30}{0.00} & \textcolor{gray!30}{0.00} & \textcolor{gray!30}{0.00} & \textcolor{gray!30}{0.00} & \textcolor{gray!30}{0.00} & 0.04 & 0.04 & \textcolor{gray!30}{0.00} & \textcolor{gray!30}{0.00} & \textcolor{gray!30}{0.00} \\
32 & \textcolor{gray!30}{0.00} & 0.30 & \textcolor{gray!30}{0.00} & \textcolor{gray!30}{0.00} & \textcolor{gray!30}{0.00} & \textcolor{gray!30}{0.00} & \textcolor{gray!30}{0.00} & \textcolor{gray!30}{0.00} & \textcolor{gray!30}{0.00} & \textcolor{gray!30}{0.00} & \textcolor{gray!30}{0.00} & \textcolor{gray!30}{0.00} \\
33 & \textcolor{gray!30}{0.00} & 0.26 & \textcolor{gray!30}{0.00} & \textcolor{gray!30}{0.00} & \textcolor{gray!30}{0.00} & \textcolor{gray!30}{0.00} & \textcolor{gray!30}{0.00} & \textcolor{gray!30}{0.00} & \textcolor{gray!30}{0.00} & \textcolor{gray!30}{0.00} & \textcolor{gray!30}{0.00} & \textcolor{gray!30}{0.00} \\
34 & \textcolor{gray!30}{0.00} & 0.23 & \textcolor{gray!30}{0.00} & \textcolor{gray!30}{0.00} & \textcolor{gray!30}{0.00} & \textcolor{gray!30}{0.00} & \textcolor{gray!30}{0.00} & \textcolor{gray!30}{0.00} & \textcolor{gray!30}{0.00} & \textcolor{gray!30}{0.00} & \textcolor{gray!30}{0.00} & \textcolor{gray!30}{0.00} \\
35 & \textcolor{gray!30}{0.00} & 0.26 & \textcolor{gray!30}{0.00} & \textcolor{gray!30}{0.00} & \textcolor{gray!30}{0.00} & \textcolor{gray!30}{0.00} & \textcolor{gray!30}{0.00} & \textcolor{gray!30}{0.00} & \textcolor{gray!30}{0.00} & \textcolor{gray!30}{0.00} & \textcolor{gray!30}{0.00} & \textcolor{gray!30}{0.00} \\
36 & \textcolor{gray!30}{0.00} & 0.10 & \textcolor{gray!30}{0.00} & 0.13 & \textcolor{gray!30}{0.00} & \textcolor{gray!30}{0.00} & \textcolor{gray!30}{0.00} & \textcolor{gray!30}{0.00} & \textcolor{gray!30}{0.00} & \textcolor{gray!30}{0.00} & \textcolor{gray!30}{0.00} & \textcolor{gray!30}{0.00} \\
37 & \textcolor{gray!30}{0.00} & \textcolor{gray!30}{0.00} & \textcolor{gray!30}{0.00} & 0.84 & 0.01 & \textcolor{gray!30}{0.00} & 0.10 & \textcolor{gray!30}{0.00} & \textcolor{gray!30}{0.00} & 0.07 & 0.01 & \textcolor{gray!30}{0.00} \\
38 & \textcolor{gray!30}{0.00} & 0.01 & \textcolor{gray!30}{0.00} & 0.54 & \textcolor{gray!30}{0.00} & \textcolor{gray!30}{0.00} & \textcolor{gray!30}{0.00} & \textcolor{gray!30}{0.00} & \textcolor{gray!30}{0.00} & \textcolor{gray!30}{0.00} & \textcolor{gray!30}{0.00} & \textcolor{gray!30}{0.00} \\
39 & \textcolor{gray!30}{0.00} & 0.05 & \textcolor{gray!30}{0.00} & 0.07 & \textcolor{gray!30}{0.00} & \textcolor{gray!30}{0.00} & \textcolor{gray!30}{0.00} & \textcolor{gray!30}{0.00} & \textcolor{gray!30}{0.00} & \textcolor{gray!30}{0.00} & \textcolor{gray!30}{0.00} & \textcolor{gray!30}{0.00} \\
40 & \textcolor{gray!30}{0.00} & 0.09 & \textcolor{gray!30}{0.00} & 0.09 & \textcolor{gray!30}{0.00} & \textcolor{gray!30}{0.00} & \textcolor{gray!30}{0.00} & \textcolor{gray!30}{0.00} & \textcolor{gray!30}{0.00} & \textcolor{gray!30}{0.00} & \textcolor{gray!30}{0.00} & \textcolor{gray!30}{0.00} \\
41 & \textcolor{gray!30}{0.00} & \textcolor{gray!30}{0.00} & 0.18 & \textcolor{gray!30}{0.00} & \textcolor{gray!30}{0.00} & \textcolor{gray!30}{0.00} & \textcolor{gray!30}{0.00} & \textcolor{gray!30}{0.00} & \textcolor{gray!30}{0.00} & \textcolor{gray!30}{0.00} & \textcolor{gray!30}{0.00} & \textcolor{gray!30}{0.00} \\
42 & \textcolor{gray!30}{0.00} & \textcolor{gray!30}{0.00} & 0.33 & \textcolor{gray!30}{0.00} & \textcolor{gray!30}{0.00} & \textcolor{gray!30}{0.00} & \textcolor{gray!30}{0.00} & \textcolor{gray!30}{0.00} & \textcolor{gray!30}{0.00} & \textcolor{gray!30}{0.00} & \textcolor{gray!30}{0.00} & \textcolor{gray!30}{0.00} \\
43 & 0.01 & \textcolor{gray!30}{0.00} & 0.24 & \textcolor{gray!30}{0.00} & \textcolor{gray!30}{0.00} & \textcolor{gray!30}{0.00} & \textcolor{gray!30}{0.00} & \textcolor{gray!30}{0.00} & \textcolor{gray!30}{0.00} & \textcolor{gray!30}{0.00} & \textcolor{gray!30}{0.00} & 0.24 \\
44 & \textcolor{gray!30}{0.00} & \textcolor{gray!30}{0.00} & 0.30 & \textcolor{gray!30}{0.00} & \textcolor{gray!30}{0.00} & \textcolor{gray!30}{0.00} & \textcolor{gray!30}{0.00} & \textcolor{gray!30}{0.00} & \textcolor{gray!30}{0.00} & \textcolor{gray!30}{0.00} & \textcolor{gray!30}{0.00} & \textcolor{gray!30}{0.00} \\
45 & \textcolor{gray!30}{0.00} & \textcolor{gray!30}{0.00} & 0.26 & \textcolor{gray!30}{0.00} & \textcolor{gray!30}{0.00} & \textcolor{gray!30}{0.00} & \textcolor{gray!30}{0.00} & \textcolor{gray!30}{0.00} & \textcolor{gray!30}{0.00} & \textcolor{gray!30}{0.00} & 0.01 & \textcolor{gray!30}{0.00} \\
46 & 0.03 & 0.01 & 0.01 & 0.02 & 0.53 & 0.01 & \textcolor{gray!30}{0.00} & \textcolor{gray!30}{0.00} & 0.20 & 0.01 & 0.03 & 0.07 \\
47 & \textcolor{gray!30}{0.00} & 0.05 & 0.07 & 0.01 & 0.23 & \textcolor{gray!30}{0.00} & \textcolor{gray!30}{0.00} & \textcolor{gray!30}{0.00} & \textcolor{gray!30}{0.00} & \textcolor{gray!30}{0.00} & 0.08 & \textcolor{gray!30}{0.00} \\
48 & \textcolor{gray!30}{0.00} & \textcolor{gray!30}{0.00} & 0.06 & 0.17 & 0.58 & \textcolor{gray!30}{0.00} & \textcolor{gray!30}{0.00} & \textcolor{gray!30}{0.00} & 0.01 & 0.10 & 0.04 & \textcolor{gray!30}{0.00} \\
49 & 0.03 & 0.01 & 0.10 & 0.02 & 0.63 & 0.02 & \textcolor{gray!30}{0.00} & \textcolor{gray!30}{0.00} & 0.07 & \textcolor{gray!30}{0.00} & 0.10 & \textcolor{gray!30}{0.00} \\
50 & 0.04 & \textcolor{gray!30}{0.00} & 0.16 & \textcolor{gray!30}{0.00} & 0.15 & 0.05 & \textcolor{gray!30}{0.00} & \textcolor{gray!30}{0.00} & 0.02 & \textcolor{gray!30}{0.00} & 0.11 & \textcolor{gray!30}{0.00} \\
51 & 0.01 & \textcolor{gray!30}{0.00} & \textcolor{gray!30}{0.00} & \textcolor{gray!30}{0.00} & 0.64 & 0.17 & 0.10 & 0.21 & 0.05 & \textcolor{gray!30}{0.00} & \textcolor{gray!30}{0.00} & 0.35 \\
52 & \textcolor{gray!30}{0.00} & \textcolor{gray!30}{0.00} & \textcolor{gray!30}{0.00} & \textcolor{gray!30}{0.00} & 0.27 & \textcolor{gray!30}{0.00} & 0.02 & 0.03 & \textcolor{gray!30}{0.00} & \textcolor{gray!30}{0.00} & \textcolor{gray!30}{0.00} & 0.40 \\
53 & \textcolor{gray!30}{0.00} & \textcolor{gray!30}{0.00} & 0.05 & \textcolor{gray!30}{0.00} & 0.25 & \textcolor{gray!30}{0.00} & 0.01 & 0.01 & \textcolor{gray!30}{0.00} & \textcolor{gray!30}{0.00} & \textcolor{gray!30}{0.00} & 0.34 \\
54 & \textcolor{gray!30}{0.00} & \textcolor{gray!30}{0.00} & \textcolor{gray!30}{0.00} & \textcolor{gray!30}{0.00} & 0.11 & 0.01 & 0.02 & 0.07 & 0.01 & \textcolor{gray!30}{0.00} & \textcolor{gray!30}{0.00} & 0.06 \\
55 & 0.03 & \textcolor{gray!30}{0.00} & \textcolor{gray!30}{0.00} & \textcolor{gray!30}{0.00} & 0.32 & 0.01 & 0.01 & \textcolor{gray!30}{0.00} & \textcolor{gray!30}{0.00} & \textcolor{gray!30}{0.00} & 0.01 & 0.31 \\
56 & 0.01 & \textcolor{gray!30}{0.00} & 0.24 & \textcolor{gray!30}{0.00} & \textcolor{gray!30}{0.00} & 0.04 & 0.13 & \textcolor{gray!30}{0.00} & 0.01 & \textcolor{gray!30}{0.00} & 0.25 & \textcolor{gray!30}{0.00} \\
57 & 0.06 & \textcolor{gray!30}{0.00} & 0.26 & \textcolor{gray!30}{0.00} & \textcolor{gray!30}{0.00} & 0.05 & 0.01 & 0.09 & 0.06 & 0.01 & 0.26 & \textcolor{gray!30}{0.00} \\
58 & \textcolor{gray!30}{0.00} & \textcolor{gray!30}{0.00} & 0.22 & \textcolor{gray!30}{0.00} & \textcolor{gray!30}{0.00} & \textcolor{gray!30}{0.00} & \textcolor{gray!30}{0.00} & \textcolor{gray!30}{0.00} & \textcolor{gray!30}{0.00} & \textcolor{gray!30}{0.00} & 0.24 & \textcolor{gray!30}{0.00} \\
59 & \textcolor{gray!30}{0.00} & \textcolor{gray!30}{0.00} & 0.31 & \textcolor{gray!30}{0.00} & \textcolor{gray!30}{0.00} & \textcolor{gray!30}{0.00} & \textcolor{gray!30}{0.00} & \textcolor{gray!30}{0.00} & \textcolor{gray!30}{0.00} & \textcolor{gray!30}{0.00} & 0.10 & \textcolor{gray!30}{0.00} \\
60 & \textcolor{gray!30}{0.00} & \textcolor{gray!30}{0.00} & 0.19 & \textcolor{gray!30}{0.00} & \textcolor{gray!30}{0.00} & \textcolor{gray!30}{0.00} & \textcolor{gray!30}{0.00} & \textcolor{gray!30}{0.00} & \textcolor{gray!30}{0.00} & \textcolor{gray!30}{0.00} & 0.04 & \textcolor{gray!30}{0.00} \\\hline
\label{tab:H1}
\end{tabular}
\end{table}

\begin{table}[!htbp]
    \centering
    \caption{Micro-cluster membership values of $\bm{H}_2$ across the three main clusters (I–III).}
    \label{tab:H2}
    \small
    \begin{tabular}{c|ccc}
    \toprule
    \multicolumn{1}{c|}{\multirow{2}{*}{\makecell{Micro\\cluster}}} & \multicolumn{3}{c}{Clusters} \\
    \cmidrule(lr){2-4}
         & I & II & III \\
    \midrule
    A & 1.04 & \textcolor{gray!30}{0.00} & \textcolor{gray!30}{0.00} \\
    B & \textcolor{gray!30}{0.00} & 1.00 & \textcolor{gray!30}{0.00} \\
    C & \textcolor{gray!30}{0.00} & \textcolor{gray!30}{0.00} & 1.00 \\
    D & \textcolor{gray!30}{0.00} & 0.05 & 0.01 \\
    E & 0.02 & \textcolor{gray!30}{0.00} & 0.19 \\
    F & \textcolor{gray!30}{0.00} & \textcolor{gray!30}{0.00} & \textcolor{gray!30}{0.00} \\
    G & 0.08 & \textcolor{gray!30}{0.00} & \textcolor{gray!30}{0.00} \\
    H & 0.06 & 0.08 & \textcolor{gray!30}{0.00} \\
    I & \textcolor{gray!30}{0.00} & \textcolor{gray!30}{0.00} & \textcolor{gray!30}{0.00} \\
    J & \textcolor{gray!30}{0.00} & 0.02 & 0.06 \\
    K & \textcolor{gray!30}{0.00} & \textcolor{gray!30}{0.00} & 0.01 \\
    L & 0.02 & \textcolor{gray!30}{0.00} & 0.01 \\
    \bottomrule
    \end{tabular}
\end{table}

\begin{table}[!htbp] 
\centering 
\caption{Full node\(\rightarrow\)community soft memberships (\(\bm{\Psi}=\bm{H}_1\bm{H}_2\), communities I–III).}
\label{tab:Psi-full}
\begin{tabular}{c|ccc} 
\hline 
\multicolumn{1}{c|}{\multirow{2}{*}{Node}} & \multicolumn{3}{c}{Clusters} \\ 
\cline{2-4} & I & II & III \\
\hline 
1 & 0.27 & \textcolor{gray!30}{0.00} & \textcolor{gray!30}{0.00} \\ 
2 & 0.26 & \textcolor{gray!30}{0.00} & 0.01 \\ 
3 & 0.22 & \textcolor{gray!30}{0.00} & 0.03 \\
4 & 0.24 & \textcolor{gray!30}{0.00} & \textcolor{gray!30}{0.00} \\ 
5 & 0.33 & 0.01 & 0.02 \\
6 & 0.10 & \textcolor{gray!30}{0.00} & \textcolor{gray!30}{0.00} \\
7 & 0.10 & \textcolor{gray!30}{0.00} & \textcolor{gray!30}{0.00} \\ 
8 & 0.08 & \textcolor{gray!30}{0.00} & \textcolor{gray!30}{0.00} \\ 
9 & 0.22 & \textcolor{gray!30}{0.00} & \textcolor{gray!30}{0.00} \\ 
10 & 0.16 & \textcolor{gray!30}{0.00} & 0.03 \\
11 & 0.13 & 0.06 & \textcolor{gray!30}{0.00} \\
12 & 0.10 & 0.05 & \textcolor{gray!30}{0.00} \\
13 & 0.20 & 0.14 & 0.03 \\
14 & 0.11 & 0.10 & \textcolor{gray!30}{0.00} \\
15 & 0.19 & 0.09 & \textcolor{gray!30}{0.00} \\
16 & 0.14 & \textcolor{gray!30}{0.00} & \textcolor{gray!30}{0.00} \\ 
17 & 0.07 & \textcolor{gray!30}{0.00} & \textcolor{gray!30}{0.00} \\ 
18 & 0.20 & \textcolor{gray!30}{0.00} & \textcolor{gray!30}{0.00} \\
19 & 0.21 & \textcolor{gray!30}{0.00} & \textcolor{gray!30}{0.00} \\
20 & 0.24 & \textcolor{gray!30}{0.00} & 0.04 \\ 
21 & \textcolor{gray!30}{0.00} & 0.08 & 0.02 \\
22 & 0.03 & 0.02 & 0.03 \\
23 & \textcolor{gray!30}{0.00} & 0.03 & 0.04 \\
24 & \textcolor{gray!30}{0.00} & 0.01 & 0.03 \\
25 & \textcolor{gray!30}{0.00} & 0.04 & 0.07 \\
26 & \textcolor{gray!30}{0.00} & 0.22 & \textcolor{gray!30}{0.00} \\
27 & \textcolor{gray!30}{0.00} & 0.22 & \textcolor{gray!30}{0.00} \\
28 & \textcolor{gray!30}{0.00} & 0.23 & \textcolor{gray!30}{0.00} \\
29 & \textcolor{gray!30}{0.00} & 0.28 & 0.02 \\
30 & \textcolor{gray!30}{0.00} & 0.20 & \textcolor{gray!30}{0.00} \\
31 & \textcolor{gray!30}{0.00} & 0.23 & \textcolor{gray!30}{0.00} \\
32 & \textcolor{gray!30}{0.00} & 0.30 & \textcolor{gray!30}{0.00} \\
33 & \textcolor{gray!30}{0.00} & 0.26 & \textcolor{gray!30}{0.00} \\
34 & \textcolor{gray!30}{0.00} & 0.23 & \textcolor{gray!30}{0.00} \\
35 & \textcolor{gray!30}{0.00} & 0.26 & \textcolor{gray!30}{0.00} \\
36 & \textcolor{gray!30}{0.00} & 0.11 & \textcolor{gray!30}{0.00} \\
37 & 0.01 & 0.04 & 0.01 \\
38 & \textcolor{gray!30}{0.00} & 0.04 & 0.01 \\ 
39 & \textcolor{gray!30}{0.00} & 0.05 & \textcolor{gray!30}{0.00} \\
40 & \textcolor{gray!30}{0.00} & 0.09 & \textcolor{gray!30}{0.00} \\
41 & \textcolor{gray!30}{0.00} & \textcolor{gray!30}{0.00} & 0.18 \\ 
42 & \textcolor{gray!30}{0.00} & \textcolor{gray!30}{0.00} & 0.33 \\ 
43 & 0.02 & \textcolor{gray!30}{0.00} & 0.24 \\ 
44 & \textcolor{gray!30}{0.00} & \textcolor{gray!30}{0.00} & 0.30 \\
45 & \textcolor{gray!30}{0.00} & \textcolor{gray!30}{0.00} & 0.26 \\
46 & 0.04 & 0.01 & 0.11 \\ 
47 & \textcolor{gray!30}{0.00} & 0.05 & 0.11 \\
48 & 0.01 & 0.01 & 0.18 \\
49 & 0.04 & 0.01 & 0.22 \\
50 & 0.04 & \textcolor{gray!30}{0.00} & 0.19 \\
51 & 0.05 & 0.02 & 0.13 \\ 
52 & 0.02 & \textcolor{gray!30}{0.00} & 0.06 \\
53 & 0.01 & \textcolor{gray!30}{0.00} & 0.10 \\
54 & 0.01 & 0.01 & 0.02 \\
55 & 0.04 & \textcolor{gray!30}{0.00} & 0.06 \\
56 & 0.02 & \textcolor{gray!30}{0.00} & 0.24 \\
57 & 0.07 & 0.01 & 0.26 \\
58 & \textcolor{gray!30}{0.00} & \textcolor{gray!30}{0.00} & 0.22 \\
59 & \textcolor{gray!30}{0.00} & \textcolor{gray!30}{0.00} & 0.31 \\
60 & \textcolor{gray!30}{0.00} & \textcolor{gray!30}{0.00} & 0.19 \\ 
\hline 
\end{tabular} 
\label{tab:Psi} 
\end{table}

\section{Appendix 3: Intersectional Multi-Attribute Fairness}\label{app:multi-attr}

\subsection{Constructing the Intersectional Fairness Matrix}

The problem of fairness often gets amplified for individuals belonging to the intersection of two or more protected/sensitive attributes (e.g. `Race', and `Gender') \cite{DBLP:conf/ijcnn/RoyKPN24}. Suppose each node has $A$ sensitive attributes. Attribute $a\in\{1,\dots,A\}$ has $m_a$ groups and one-hot indicator $G^{(a)}\in\mathbb{R}^{n\times m_a}$.
To enforce \emph{intersectional} (joint) parity across the Cartesian product of all attributes as previously done in~\cite{DBLP:conf/fat/RoyH, DBLP:journals/corr/abs-2509-08156}, we build a joint one-hot matrix:
\begin{equation}
G_{\text{int}} \;\in\; \mathbb{R}^{n\times M},\qquad
M \;=\; \prod_{a=1}^{A} m_a,
\end{equation}

whose columns correspond to all joint categories (e.g., \textit{male$\wedge$Asian}, \textit{female$\wedge$White}, \dots). Concretely, each joint column is the elementwise AND (Hadamard product) of one columns taken from each $G^{(a)}$:
\begin{equation}
G_{\text{int}} \;=\; G^{(1)} \,\odot\, G^{(2)} \,\odot\, \cdots \,\odot\, G^{(A)}.
\end{equation}
We then apply proportional centering and drop one redundant column to avoid linear dependence:
\begin{equation}
F_{\text{int}} \;=\; G_{\text{int}} \;-\; \tfrac{1}{n}\,\bm{1}_n \big(\bm{1}_n^\top G_{\text{int}}\big),
\qquad
F_{\text{int}}\in\mathbb{R}^{n\times(M-1)}.
\end{equation}
DFNMF’s fairness penalty becomes
\begin{equation}
\mathcal{R}_{\text{int}}(H) \;=\; \big\| F_{\text{int}}^\top H \big\|_F^2,
\end{equation}
which enforces demographic parity \emph{jointly} over all intersections. Complexity-wise, this adds $O(nk\,M)$ per iteration (for forming $F_{\text{int}}^\top H$). In our SBM study with two attributes (gender: $m_1{=}2$, ethnicity: $m_2{=}5$), $M{=}10$ is modest.

\paragraph{Tiny schematic (two attributes).}
For gender $\{\mathrm{M},\mathrm{F}\}$ and ethnicity $\{\mathrm{A},\mathrm{W},\mathrm{B},\mathrm{C},\mathrm{D}\}$, the $M{=}10$ joint groups are
\[
\begin{aligned}
\{&(\mathrm{M},\mathrm{A}),(\mathrm{M},\mathrm{W}),(\mathrm{M},\mathrm{B}),(\mathrm{M},\mathrm{C}),(\mathrm{M},\mathrm{D}),\\
 &(\mathrm{F},\mathrm{A}),(\mathrm{F},\mathrm{W}),(\mathrm{F},\mathrm{B}),(\mathrm{F},\mathrm{C}),(\mathrm{F},\mathrm{D})\}.
\end{aligned}
\]
Each column of $G_{\text{int}}$ is the indicator of one joint group; $F_{\text{int}}$ is its centered version (with one dropped column).

\subsection{Metrics under Intersectional Fairness}
We evaluate standard utility and fairness metrics:
\begin{itemize}
  \item \textbf{Modularity} $Q$ (higher is better).
  \item \textbf{Intersectional balance} $\bar B_{\text{int}}$ computed over the $M$ joint groups: 
  $\bar B_{\text{int}}=\tfrac{1}{k}\sum_{l=1}^{k} \min_{g\neq g'} \tfrac{|V_g\cap C_l|}{|V_{g'}\cap C_l|}$, where $g,g'$ range over joint categories.
\end{itemize}

\subsection{SBM at Scale: Enforcing Intersectional Parity}\label{app:multi-sbm}

\begin{table}[!hbt]
\centering
\setlength{\tabcolsep}{6pt}   
\renewcommand{\arraystretch}{1.1} 
\caption{\textbf{SBM, intersectional multi-attribute fairness} ($k{=}10$, $\lambda{=}100$). Intersectional uses $M{=}10$ joint groups (gender$\times$ethnicity).}
\label{tab:multi_sbm_intersectional}
\begin{tabular}{l c c c c}
\toprule
\textbf{Attribute} & \textbf{$n$} & M (\#groups) &  $\mathbf{Q}$ & $\mathbf{\bar B}$ \\
\midrule
\multirow{3}{*}{Gender only} & 2K & 2 & 1.000 & 1.000 \\
                             & 5K & 2 & 1.000 & 1.000 \\
                             & 10K & 2 & 1.000 & 1.000 \\
\midrule
\multirow{3}{*}{Ethnicity only} & 2K & 5 & 0.119 & 0.8985 \\
                                & 5K & 5 & 0.124 & 0.9999 \\
                                & 10K & 5 & 0.124 & 1.0000 \\
\midrule
\multirow{3}{*}{Intersectional (G$\times$E)} & 2K & 10 & 0.110 & 0.3522 \\
                                             & 5K & 10 & 0.112 & 0.4263 \\
                                             & 10K & 10 & 0.115 & 0.4507 \\
\bottomrule
\end{tabular}
\end{table}

We compare \emph{single-attribute} fairness (gender-only; ethnicity-only) against \emph{intersectional} fairness (gender$\times$ethnicity) on SBMs with $k{=}10$ clusters and $n\in\{2\mathrm{K},5\mathrm{K},10\mathrm{K}\}$ nodes. We fix $\lambda{=}100$ to isolate the effect of the constraint. For single-attribute runs, $\bar B$ is computed on that attribute; for the intersectional run, $\bar B_{\text{int}}$ is computed on the $M{=}10$ joint groups. The results are illustrated in Table~\ref{tab:multi_sbm_intersectional}.

\paragraph{Discussion}
\begin{itemize}
  \item \textbf{Single-attribute parity}. Gender-only aligns with the planted structure in this SBM (both $Q$ and $\bar B$ near 1). Ethnicity-only parity achieves very high balance with modest $Q{\approx}0.12$, improving or stabilizing as $n$ grows.
  \item \textbf{Intersectional parity}. Enforcing joint parity across $M{=}10$ groups is stricter: $\bar B_{\text{int}}$ is lower than single-attribute balances at small $n$, but \emph{increases with scale} (from $0.35$ at $2\mathrm{K}$ to $0.45$ at $10\mathrm{K}$). The utility cost is controlled (slight $Q$ increase with $n$ from $0.110$ to $0.115$), indicating that larger graphs make intersectional constraints more feasible.
  \item \textbf{Takeaway}. DFNMF can efficiently enforce intersectional demographic parity directly via $F_{\text{int}}$; The soft balanced fairness encoding introduced in Section~\ref{sec: bal_constraint} allows this extension efficiently. The expected trade-off (tighter fairness $\Rightarrow$ lower $Q$) is observed, while scaling the graph improves joint feasibility without tuning $\lambda$.
\end{itemize}

\end{document}